\documentclass[a4paper]{article}   
    
\usepackage[sort&compress,comma,authoryear]{natbib}
\usepackage{subfigure}
\usepackage[english]{babel}
\usepackage[utf8x]{inputenc}
\usepackage[T1]{fontenc}
  
\usepackage[a4paper,top=3cm,bottom=2cm,left=3cm,right=3cm,marginparwidth=1.75cm]{geometry}

\usepackage{amsmath}
\usepackage{graphicx}
\usepackage[colorinlistoftodos]{todonotes}
\usepackage[colorlinks=true, allcolors=blue]{hyperref}

\usepackage{authblk} 
\usepackage{setspace}
\usepackage{amsfonts}
\usepackage{bm}
\usepackage{xspace}
\usepackage{fullpage}
\usepackage{dsfont}
\usepackage{algorithm,rotating}
\usepackage{graphicx} 
\usepackage{url}
\usepackage{enumerate}
\usepackage{xcolor,soul}
\usepackage{thm-restate}
\usepackage[small,bf]{caption}
\usepackage{standalone}
\usepackage{verbatim}
\usepackage{setspace}
\usepackage{paralist}
\usepackage{mdwmath}
\usepackage{amssymb}
\usepackage{amsmath}
\usepackage{amsthm}
\usepackage{xfrac}
\usepackage{mathtools}  
\usepackage{chngcntr}
\usepackage{tabulary}
\usepackage{booktabs}
\usepackage{epsfig}
\usepackage{wrapfig}
\usepackage{lipsum}
\usepackage{caption}
\usepackage{multirow}
\usepackage{algcompatible} 
\usepackage{footnote}
\usepackage{paralist}
\usepackage{hyperref}

\newtheorem{theorem}{Theorem}
\newtheorem{definition}{Definition}

\newtheorem{remark}{Remark}
\newtheorem{lemma}{Lemma}

\newtheorem{cor}{Corollary}

\newtheorem{proposition}{Proposition}
\newtheorem{observation}{Observation}

\newtheorem{asm}{Assumption}

\newcommand{\indi}{\mathds{1}}

\newcommand{\R}{\mathbb{R}}

\newcommand{\var}{\operatorname{Var}}

\newcommand{\vecx}{\mathbf{x}}
\newcommand{\vecw}{\mathbf{w}}

\newcommand{\vecv}{\mathbf{v}}

\newcommand{\vecy}{\mathbf{y}}
\newcommand{\vecg}{\mathbf{g}}

\newcommand{\vece}{\mathbf{e}}
\newcommand{\vecz}{\mathbf{z}}
\newcommand{\vecp}{\mathbf{p}}

\newcommand{\matA}{\mathbf{A}}

\newcommand{\matX}{\mathbf{X}}
\newcommand{\matH}{\mathbf{H}}
\newcommand{\matU}{\mathbf{U}}
\newcommand{\matI}{\mathbf{I}}
\newcommand{\gradf}{\nabla f}
\newcommand{\gradh}{\nabla h}

\newcommand{\gradF}{\nabla F}

\newcommand{\bigo}{\mathcal{O}}

\newcommand{\W}{\mathcal{W}}
\newcommand{\vect}[1]{\mathbf{#1}}
\newcommand{\mat}[1]{\mathbf{#1}}

\newcommand{\norms}[1]{\|#1\|}

\newcommand{\twonms}[1]{\|#1\|_2}

\newcommand{\onenms}[1]{\|#1\|_1}

\newcommand{\fbnorms}[1]{\|#1\|_F}
\newcommand{\prob}[1]{\mathbb{P}\left\{#1\right\}}
\newcommand{\probs}[1]{\mathbb{P}\{#1\}}

\newcommand{\EE}{\ensuremath{\mathbb{E}}}
\newcommand{\EXPS}[1]{\mathbb{E}[#1]}

\newcommand{\innerps}[2]{\langle#1,#2\rangle}
\newcommand{\abs}[1]{\left|#1\right|}
\newcommand{\abss}[1]{|#1|}
\newcommand{\tsp}{^{\rm T}}

\newcommand{\B}{\mathcal{B}}
\newcommand{\D}{\mathcal{D}}
\newcommand{\U}{\mathcal{U}}
\newcommand{\T}{\mathcal{T}}
\newcommand{\M}{\mathcal{M}}
\newcommand{\med}{{\sf med}}
\newcommand{\trim}{{\sf trmean}}
\newcommand{\Z}{\mathcal{Z}}
\newcommand{\vecmu}{\boldsymbol\mu}
\newcommand{\N}{\mathcal{N}}

\algblockdefx{PARFOR}{ENDPARFOR}[1]%
  {\textbf{for all }#1 \textbf{do in parallel}}%
  {\textbf{end for}}

\definecolor{darkgreen}{rgb}{0.0, 0.6, 0.2}
\definecolor{lightgray}{rgb}{0.8, 0.85, 0.85}

\title{Byzantine-Robust Distributed Learning: Towards Optimal Statistical Rates}
\author[1]{Dong Yin \thanks{dongyin@berkeley.edu}}
\author[3]{Yudong Chen \thanks{yudong.chen@cornell.edu}}
\author[1]{Kannan Ramchandran \thanks{kannanr@berkeley.edu}}
\author[1,2]{Peter Bartlett \thanks{peter@berkeley.edu}}
\affil[1]{Department of Electrical Engineering and Computer Sciences, UC Berkeley}
\affil[2]{Department of Statistics, UC Berkeley}
\affil[3]{School of Operations Research and Information Engineering, Cornell University}  

\begin{document}
\maketitle

\begin{abstract}
In large-scale distributed learning, security issues have become increasingly important. Particularly in a decentralized environment, some computing units may behave abnormally, or even exhibit Byzantine failures---arbitrary and potentially adversarial behavior.
In this paper, we develop distributed learning algorithms that are provably robust against such failures, with a focus on achieving optimal statistical performance. 
A main result of this work is a sharp analysis of two robust
distributed gradient descent algorithms based on median and trimmed
mean operations, respectively. We prove statistical error rates for
three kinds of population loss functions: strongly convex,
non-strongly convex, and smooth non-convex. In particular, these algorithms are shown to achieve order-optimal statistical error rates for strongly convex losses. To achieve better communication efficiency, we further propose a median-based distributed algorithm that is provably robust, and uses only one communication round. For strongly convex quadratic loss, we show that this algorithm achieves the same optimal error rate as the robust distributed gradient descent algorithms.
\end{abstract}

\section{Introduction}
\label{sec:intro}

Many tasks in computer vision, natural language processing and recommendation systems require learning complex prediction rules from large datasets. As the scale of the datasets in these learning tasks continues to grow, it is crucial to utilize the power of distributed computing and storage. In such large-scale distributed systems, robustness and security issues have become a major concern. In particular, individual computing units---known as worker machines---may exhibit abnormal behavior due to crashes, faulty hardware, stalled computation or unreliable communication channels. Security issues are only exacerbated in the so-called \emph{Federated Learning} setting, a modern distributed learning paradigm that is more decentralized, and that uses the data owners' devices (such as mobile phones and personal computers) as worker machines~\citep{mcmahan2017federated,konevcny2016federatedcommunication}. Such machines are often more unpredictable, and in particular may be susceptible to malicious and coordinated attacks. 

Due to the inherent unpredictability of this abnormal (sometimes
adversarial) behavior, it is typically modeled as \emph{Byzantine
failure}~\citep{lamport1982byzantine}, meaning that some worker
machines may behave completely arbitrarily and can send any message to
the master machine that maintains and updates an estimate of the
parameter vector to be learned. Byzantine failures can incur major
degradation in learning performance. It is well-known that standard
learning algorithms based on naive aggregation of the workers'
messages can be arbitrarily skewed by a single Byzantine-faulty
machine. Even when the messages from Byzantine machines take only
moderate values---and hence are difficult to detect---and when the number of such machines is small, the performance loss can still be significant. We demonstrate such an example in our experiments in Section~\ref{sec:experiments}.

In this paper, we aim to develop distributed statistical learning algorithms that are provably robust against Byzantine failures. While this objective is considered in a few recent works~\citep{feng2014distributed,blanchard2017byzantine,chen2017distributed}, a fundamental problem remains poorly understood, namely the \emph{optimal statistical performance} of a robust learning algorithm. A learning scheme in which the master machine always outputs zero regardless of the workers' messages is certainly not affected by Byzantine failures, but it will not return anything statistically useful either. On the other hand, many standard distributed algorithms that achieve good statistical performance in the absence of Byzantine failures, become completely unreliable otherwise.
Therefore, a main goal of this work is to understand the following
questions: what is the best achievable statistical performance while
being Byzantine-robust, and what algorithms achieve this performance?

To formalize this question, we consider a standard statistical setting of empirical risk minimization (ERM). Here $ nm $ data points are sampled independently from some distribution and distributed evenly among $ m $ machines,  $ \alpha m $ of which are Byzantine. The goal is to learn a parametric model by minimizing some loss function defined by the data.
In this statistical setting, one expects that the error in learning the parameter, measured in an appropriate metric, should decrease when the amount of data $ nm $ becomes larger and the fraction of Byzantine machines $ \alpha $ becomes smaller. In fact, we can show that, at least for strongly convex problems, no algorithm can achieve an error lower than
\begin{align*}
\widetilde{\Omega} \left( \frac{\alpha}{\sqrt{n}} +
\frac{1}{\sqrt{nm}} \right)= \widetilde{\Omega} \left(
\frac{1}{\sqrt{n}} \Big( \alpha + \frac{1}{\sqrt{m}} \Big) \right),
\end{align*}
regardless of communication costs;\footnote{Throughout the paper, unless otherwise stated, $ \Omega(\cdot) $ and $ \bigo(\cdot) $ hide universal multiplicative constants; $ \widetilde{\Omega} (\cdot) $ and $ \widetilde{\bigo}(\cdot)  $ further hide terms that are independent of $ \alpha, n, m $ or logarithmic in $ n,m. $} see Observation~\ref{obs:lower-bound} in Section~\ref{sec:lower-bound}.
Intuitively, the above error rate is the optimal rate that one should target, as $ \frac{1}{\sqrt{n}} $ is the effective standard deviation for each machine with $ n $ data points, $ \alpha $ is the bias effect of Byzantine machines, and $ \frac{1}{\sqrt{m}} $ is the averaging effect of $ m $ normal machines. When there are no or few Byzantine machines, we see the usual scaling $ \frac{1}{\sqrt{mn}} $ with the total number of data points; when some machines are Byzantine, their influence remains bounded, and moreover is proportional to $ \alpha $.
If an algorithm is guaranteed to attain this bound, we are assured
that we do not sacrifice the quality of learning when trying to guard
against Byzantine failures---we pay a price that is unavoidable, but
otherwise we achieve the best possible statistical accuracy in the presence of Byzantine failures. 
%We also note that in this work, we do not target obtaining the optimal dependence on $d$, which is
%on the dimension of the parameter to be learned, and we propose that obtaining better dependence 

Another important consideration for us is \emph{communication efficiency}. As communication between machines is costly, one cannot simply send all data to the master machine. This constraint precludes direct application of standard robust learning algorithms (such as M-estimators~\citep{huber2011robust}), which assume access to all data. Instead, a desirable algorithm should involve a small number of communication rounds as well as a small amount of data communicated per round. We consider a setting where in each round a worker or master machine can only communicate a vector of size $\bigo(d)$, where $ d $ is the dimension of the parameter to be learned.
In this case, the total communication cost is proportional to the number of communication rounds. 

To summarize, we aim to develop distributed learning algorithms that simultaneously achieve two objectives:
\begin{list}{\labelitemi}{\leftmargin=1em}
\item \textbf{Statistical optimality:} attain an $\widetilde{\bigo}(\frac{\alpha}{\sqrt{n}} + \frac{1}{\sqrt{nm}}) $ rate. 
\item \textbf{Communication efficiency:} $ \bigo(d) $ communication per round, with as few rounds as possible.
\end{list}
To the best of our knowledge, \emph{no existing algorithm achieves these two goals simultaneously}.
In particular, previous robust algorithms either have unclear or
sub-optimal statistical guarantees, or incur a high communication cost
and hence are not applicable in a distributed setting---we discuss related work in more detail in Section~\ref{sec:related-work}. 

\subsection{Our Contributions}
We propose two robust distributed gradient descent (GD) algorithms, one based on coordinate-wise median, and the other on coordinate-wise trimmed mean. We establish their statistical error rates for strongly convex, non-strongly convex, and non-convex \emph{population} loss functions. For strongly convex losses, we show that these algorithms achieve order-optimal statistical rates under mild  conditions. We further propose a median-based robust algorithm that only requires one communication round, and show that it also achieves the optimal rate for strongly convex quadratic losses.  The statistical error rates of these three algorithms are summarized as follows.
\begin{list}{\labelitemi}{\leftmargin=1em \itemsep=-1mm}
\item \textbf{Median-based GD:} $\widetilde{\bigo}(\frac{\alpha}{\sqrt{n}} + \frac{1}{\sqrt{nm}} + \frac{1}{n})$, order-optimal for strongly convex loss if $n \gtrsim m$.
\item \textbf{Trimmed-mean-based GD:} $\widetilde{\bigo}(\frac{\alpha}{\sqrt{n}} + \frac{1}{\sqrt{nm}})$, order-optimal for strongly convex loss.
\item \textbf{Median-based one-round algorithm:} $\widetilde{\bigo}(\frac{\alpha}{\sqrt{n}} + \frac{1}{\sqrt{nm}} + \frac{1}{n})$, order-optimal for strongly convex quadratic loss if $n \gtrsim m$.
\end{list}
A major technical challenge in our statistical setting here is as
follows: the $ nm $ data points are sampled once and fixed, and each
worker machine has access to a fixed set of data throughout the learning process. This creates complicated probabilistic dependency across the iterations of the algorithms. Worse yet, the Byzantine machines, which have complete knowledge of the data and the learning algorithm used, may create further unspecified probabilistic dependency. We overcome this difficulty by proving certain \emph{uniform} bounds via careful covering arguments.
Furthermore, for the analysis of median-based algorithms, we cannot simply adapt standard techniques (such as those in~\citet{minsker2015geometric}), which can only show that the output of the master machine is as accurate as that of \emph{one} normal machine, leading to a sub-optimal $ \bigo(\frac{1}{\sqrt{n}}) $ rate even without Byzantine failures ($ \alpha = 0 $). Instead, we make use of a more delicate argument based on normal approximation and Berry-Esseen-type inequalities, which allows us to achieve the better $ \bigo(\frac{1}{mn}) $ rates when $ \alpha $ is small while being robust for a nonzero $ \alpha $.

Above we have omitted the dependence on the parameter dimension $ d. $; see our main theorems for the precise results. In some settings the rates in these results may not have the optimal dependence on $d$. Understanding the fundamental limits of robust distributed learning in high dimensions, as well as developing algorithms with optimal dimension dependence, is an interesting and important future direction.

% Remove the table for now.
\iffalse
\begin{table}[htbp]
\centering
\begin{tabular}{|c|c|c|c|}
\hline
Algorithm & GD, median & GD, trimmed mean & one-round, median \\  \hline
Error rate & $\widetilde{\bigo}(\frac{\alpha}{\sqrt{n}} + \frac{1}{\sqrt{nm}} + \frac{1}{n})$ & $\widetilde{\bigo}(\frac{\alpha}{\sqrt{n}} + \frac{1}{\sqrt{nm}})$ & $\widetilde{\bigo}(\frac{\alpha}{\sqrt{n}} + \frac{1}{\sqrt{nm}} + \frac{1}{n})$ \\ \hline
\multirow{2}{*}{Function class} & (strongly) convex,  & (strongly) convex,  & \multirow{2}{*}{strongly convex quadratic}  \\ 
& smooth non-convex & smooth non-convex & \\ \hline
\end{tabular}
\caption{Main results on the error rate of the three algorithms. GD = gradient descent.}
\label{tab:main-results}
\end{table} 
\fi

\subsection{Notation}
We denote vectors by boldface lowercase letters such as $\vecw$, and
the elements in the vector are denoted by italics letters with
subscripts, such as $w_k$. Matrices are denoted by boldface uppercase
letters such as $\matH$. For any positive integer $N$, we denote the
set $\{1,2,\ldots, N\}$ by $[N]$. For vectors, we denote the $\ell_2$
norm and $\ell_\infty$ norm by $\twonms{\cdot}$ and
$\|\cdot\|_\infty$, respectively. For matrices, we denote the operator
norm and the Frobenius norm by $\twonms{\cdot}$ and $\fbnorms{\cdot}$,
respectively. We denote by $\Phi(\cdot)$ the cumulative distribution
function (CDF) of the standard Gaussian distribution. For any differentiable function $f:\R^d\rightarrow\R$, we denote its partial derivative with respect to the $k$-th argument by $\partial_kf$.

\section{Related Work}\label{sec:related-work}

Outlier-robust estimation in non-distributed settings is a classical
topic in statistics~\citep{huber2011robust}. Particularly relevant to
us is the so-called \emph{median-of-means} method, in which one
partitions the data into $m$ subsets, computes an estimate from each
subset, and finally takes the median of these $m$ estimates. This idea
is studied
in~\citet{nemirovskii1983problem,jerrum1986random,alon1999space,lerasle2011robust,minsker2015geometric},
and has been applied to bandit and least square regression
problems~\citep{bubeck2013bandits, lugosi2016risk,kogler2016efficient}
as well as problems involving heavy-tailed
distributions~\citep{hsu2016loss,lugosi2017sub}. In a very recent
work, \citet{minsker2017distributed} provide a new analysis of
median-of-means using a normal approximation. We borrow some techniques from this paper, but need to address a significant harder problem: 1) we deal with the Byzantine setting with arbitrary/adversarial outliers, which is not considered in their paper; 2) we study iterative algorithms for general multi-dimensional problems with convex and non-convex losses, while they mainly focus on one-shot algorithms for mean-estimation-type problems.

The median-of-means method is used in the context of Byzantine-robust distributed learning in two recent papers.
In particular, the work of~\citet{feng2014distributed} considers a simple one-shot application of median-of-means, and only proves a sub-optimal $\widetilde{\bigo}(\frac{1}{\sqrt{n}})$ error rate as mentioned.
The work of~\citet{chen2017distributed} considers only strongly convex losses, and seeks to circumvent the above issue by grouping the worker machines into mini-batches; however,  their rate $\widetilde{\bigo}(\frac{\sqrt{\alpha}}{\sqrt{n}} + \frac{1}{\sqrt{nm}})$ still falls short of being optimal, and in particular their algorithm fails even when there is only one Byzantine machine in each mini-batch.

Other methods have been proposed for Byzantine-robust distributed learning and optimization; e.g.,~\citet{su2016fault,su2016non}. These works consider optimizing fixed functions and do not provide guarantees on statistical error rates. Most relevant is the work by~\citet{blanchard2017byzantine}, who propose to aggregate the gradients from worker machines using a robust procedure. 
Their optimization setting---which is at the level of stochastic gradient descent and assumes unlimited, independent access to a strong stochastic gradient oracle---is fundamentally different from ours; in particular, they do not provide a characterization of the statistical errors given a fixed number of data points.

Communication efficiency has been studied extensively in non-Byzantine distributed settings~\citep{mcmahan2016communication,yin2017gradient}. 
An important class of algorithms are based on one-round aggregation methods~\citep{zhang2012communication,zhang2015divide,rosenblatt2016optimality}. More sophisticated algorithms have been proposed in order to achieve better accuracy than the one-round approach while maintaining lower communication costs; examples include DANE~\citep{shamir2014communication}, Disco~\citep{zhang2015disco}, distributed SVRG~\citep{lee2015distributed} and their variants~\citep{reddi2016aide,wang2017giant}. Developing Byzantine-robust versions of these algorithms is an interesting future direction.

For outlier-robust estimation in non-distributed settings, much progress has been made recently in terms of improved performance in high-dimensional problems~\citep{diakonikolas2016robust,lai2016agnostic,bhatia2015robust} as well as developing list-decodable and semi-verified learning schemes when a majority of the data points are adversarial~\citep{charikar2017learning}. These results are not directly applicable to our distributed setting with general loss functions, but it is nevertheless an interesting future problem to investigate their potential extension for our problem.

\section{Problem Setup}

In this section, we formally set up our problem and introduce a few concepts key to our the algorithm design and analysis. Suppose that training data points are sampled from some unknown distribution $\D$ on the sample space $\Z$. Let $f(\vecw;\vecz)$ be a loss function of a parameter vector $\vecw\in\W\subseteq\R^d$ associated with the data point $\vecz$, where $\W$ is the parameter space, and  $F(\vecw):=\EE_{\vecz\sim\D}[f(\vecw;\vecz)]$ is the corresponding population loss function. Our goal is to learn a model defined by the parameter that minimizes the population loss:
\begin{equation}\label{eq:min-loss}
\vecw^* = \arg\min_{\vecw\in\W} F(\vecw).
\end{equation}
The parameter space $\W$ is assumed to be convex and compact with diameter $D$, i.e., $\twonms{\vecw-\vecw'}\le D, \forall \vecw,\vecw'\in\W$. 
We consider a distributed computation model with one master machine and $m$ worker machines. Each worker machine stores $n$ data points, each of which is sampled independently from $ \D $. Denote by $\vecz^{i,j}$ the $j$-th data on the $i$-th worker machine, and $F_i(\vecw) := \frac{1}{n}\sum_{j=1}^n f(\vecw;\vecz^{i,j})$ the empirical risk function for the $i$-th worker. We assume that an $\alpha$ fraction of the $ m $ worker machines are Byzantine, and the remaining $1-\alpha$ fraction are normal.  With the notation $[N]:=\{1,2,\ldots, N\}$, we index the set of worker machines by $ [m] $, and denote the set of Byzantine machines by $\B \subset [m]$ (thus $|\B| = \alpha m$). 
The master machine communicates with the worker machines using some predefined protocol. The Byzantine machines need not obey this protocol and can send arbitrary messages to the master; in particular, they may have complete knowledge of the system and learning algorithms, and can collude with each other. 

We introduce the coordinate-wise median and trimmed mean operations, which serve as building blocks for our algorithm.
\begin{definition}[Coordinate-wise median]\label{def:median}
	For vectors $\vecx^i \in \mathbb{R}^d$, $i \in [m]$, the coordinate-wise median $\vecg := \med\{\vecx^i:i\in[m]\}$
	is a vector with its $ k $-th coordinate being $g_k = \med\{x^i_k:i\in[m]\}$ for each $k\in[d]$, where $ \med $ is the usual (one-dimensional) median.
\end{definition}
\begin{definition}[Coordinate-wise trimmed mean]\label{def:trimmed-mean}
	For $\beta\in[0,\frac{1}{2})$ and vectors $\vecx^i\in \mathbb{R}^d$, $i \in [m]$, the coordinate-wise $ \beta $-trimmed mean $
	\vecg := \trim_\beta\{\vecx^i:i\in[m]\}$ is a vector with its $ k $-th coordinate being $g_k = \frac{1}{(1-2\beta)m}\sum_{x \in U_k}x$ for each $k\in[d]$. Here $U_k$ is a subset of $\{x^1_k,\ldots, x^m_k\}$ obtained by removing the largest and smallest $\beta$ fraction of its elements.
\end{definition}

For the analysis, we need several standard definitions concerning random variables/vectors.
\begin{definition}[Variance of random vectors]\label{def:vect-var}
For a random vector $\vecx$, define its variance as $\var(\vecx) := \EXPS{\twonms{\vecx - \EXPS{\vecx}}^2}$.
\end{definition}
\begin{definition}[Absolute skewness]\label{def:abs-skewness}
For a one-dimensional random variable $X$, define its absolute skewness\footnote{Note the difference with the usual skewness $\frac{\EXPS{(X - \EXPS{X})^3}}{\var(X)^{3/2}}$.} as $
\gamma(X) := \frac{\EXPS{|X - \EXPS{X}|^3}}{\var(X)^{3/2}}$.
For a $d$-dimensional random vector $\vecx$, we define its absolute skewness as the vector of the absolute skewness of each coordinate of $\vecx$, i.e., $ \gamma(\vecx) := [\gamma(x_1)~\gamma(x_2)~\cdots~\gamma(x_d)]^\top$.
\end{definition}
\begin{definition}[Sub-exponential random variables]\label{def:sub-exponential} A random variable $X$ with $\EXPS{X}=\mu$ is called $v$-sub-exponential if $
\EXPS{e^{\lambda(X-\mu)}} \le e^{\frac{1}{2}v^2\lambda^2},~\forall~ | \lambda | <\frac{1}{v}$.
\end{definition}

Finally, we need several standard concepts from convex analysis regarding a differentiable function $h(\cdot):\R^d\rightarrow\R$.
\begin{definition}[Lipschitz]\label{def:lipschitz}
$ h $ is $ L $-Lipschitz if $|h(\vecw) - h(\vecw')| \le L\twonms{\vecw - \vecw'}, \forall~\vecw,\vecw'$.
\end{definition}
\begin{definition}[Smoothness]\label{def:smoothness}
$ h $ is $ L' $-smooth if $\twonms{\gradh(\vecw) - \gradh(\vecw')} \le L'\twonms{\vecw - \vecw'}, \forall~\vecw,\vecw'$.
\end{definition}
\begin{definition}[Strong convexity]\label{def:strong-cvx}
$ h $ is $ \lambda $-strongly convex if $ h(\vecw') \ge h(\vecw) + \innerps{\gradh(\vecw)}{\vecw' - \vecw} + \frac{\lambda}{2}\twonms{\vecw' - \vecw}^2, \forall~\vecw,\vecw'$.
\end{definition}

\section{Robust Distributed Gradient Descent}\label{sec:gradient-descent}

We describe two robust distributed gradient descent algorithms, one based on coordinate-wise median and the other on trimmed mean. These two algorithms are formally given in Algorithm~\ref{alg:robust-gd} as Option I and Option II, respectively, where the symbol $*$ represents an arbitrary vector. 

In each parallel iteration of the algorithms, the master machine broadcasts the current model parameter to all worker machines. The normal worker machines compute the gradients of their local loss functions and then send the gradients back to the master machine. The Byzantine machines may send any messages of their choices. The master machine then performs a gradient descent update on the model parameter with step-size $\eta$, using either the coordinate-wise median or trimmed mean of the received gradients. The Euclidean projection $\Pi_\W(\cdot)$ ensures that the model parameter stays in the parameter space $ \W $. 

\begin{algorithm}[h]
  \caption{Robust Distributed Gradient Descent}\label{alg:robust-gd}
  \begin{algorithmic}
   \REQUIRE Initialize parameter vector $\vecw^0\in\W$, algorithm parameters $\beta$ (for Option II),  $\eta$ and $ T $.
   \FOR{$t=0,1,2,\ldots, T-1$}
   \STATE \textit{\underline{Master machine}}: send $\vecw^t$ to all the worker machines.
   \PARFOR{$i\in[m]$}
   \STATE \textit{\underline{Worker machine $i$}}: compute local gradient
   \STATE
   \begin{align*}
   \vecg^i(\vecw^t) \leftarrow \begin{cases}
   \gradF_i(\vecw^t) & \text{normal worker machines}, \\
   * & \text{Byzantine machines}, \end{cases}
   \end{align*}
   \STATE send $\vecg^i(\vecw^t)$ to master machine.
   \ENDPARFOR
   \STATE \textit{\underline{Master machine}}: compute aggregate gradient
   \begin{align*} 
   \vecg(\vecw^t) \leftarrow \begin{cases}
   \med\{\vecg^i(\vecw^t) : i\in[m]\} & \text{ Option I}\\
   \trim_\beta\{\vecg^i(\vecw^t) : i\in[m]\} & \text{ Option II}
   \end{cases}
   \end{align*}
   \STATE update model parameter $\vecw^{t+1} \leftarrow \Pi_\W(\vecw^t - \eta \vecg(\vecw^t))$.
   \ENDFOR
  \end{algorithmic}
\end{algorithm}

Below we provide statistical guarantees on the error rates of these algorithms, and compare their performance. Throughout we assume that each loss function $f(\vecw;\vecz)$ and the population loss function $F(\vecw)$ are smooth: 
\begin{asm}[Smoothness of $f$ and $F$]\label{asm:smoothness}
	For any $\vecz\in\Z$, the partial derivative of $f(\cdot;\vecz)$ with respect to the $ k $-th coordinate of its first argument, denoted by $\partial_kf(\cdot;\vecz)$, is $L_k$-Lipschitz for each $k\in[d]$, and the function $f(\cdot;\vecz)$ is $L$-smooth. Let $\widehat{L} := \sqrt{\sum_{k=1}^d L_k^2}$. Also assume that the population loss function $F(\cdot)$ is $L_F$-smooth.
\end{asm}
It is easy to see hat $ L_F \le L \le \widehat{L} $. When the dimension of $\vecw$ is high, the quantity $\widehat{L}$ may be large. However, we will soon see that $\widehat{L}$ only appears in the logarithmic factors in our bounds and thus does not have a significant impact.

\subsection{Guarantees for Median-based Gradient Descent}\label{sec:med-convergence}

We first consider our median-based algorithm, namely  Algorithm~\ref{alg:robust-gd} with Option I. We impose the assumptions that the gradient of the loss function $ f $ has bounded variance, and each coordinate of the gradient has coordinate-wise bounded absolute skewness:
\begin{asm}[Bounded variance of gradient]\label{asm:bounded-variance} For any $\vecw\in\W$, $\var( \gradf(\vecw;\vecz) ) \le V^2$.
\end{asm}
\begin{asm}[Bounded skewness of gradient]\label{asm:bounded-skewness} For any $\vecw\in\W$, $\|\gamma(\gradf(\vecw;\vecz))\|_\infty \le S$.
\end{asm}

These assumptions are satisfied in many learning problems with small values of $ V^2 $ and $ S $. Below we provide a concrete example in terms of a linear regression problem.
\begin{proposition}\label{thm:eg-var-skew-rad}
Suppose that each data point $\vecz = (\vecx, y) \in \R^d \times \R$ is generated by $ y = \vecx^\top\vecw^* + \xi$
with some $\vecw^*\in\W$. Assume that the elements of $\vecx$ are independent and uniformly distributed in $\{-1,1\}$, and that the noise $\xi \sim \mathcal{N}(0, \sigma^2)$ is independent of $\vecx$. With the quadratic loss function $f(\vecw; \vecx, y) = \frac{1}{2}(y - \vecx\tsp\vecw)^2$, we have $ \var(\gradf(\vecw; \vecx, y)) = (d-1) \twonms{\vecw - \vecw^*}^2 + d\sigma^2$, and
$
\|\gamma(\gradf(\vecw;\vecx, y))\|_\infty \le 480.
$
\end{proposition}
We prove Proposition~\ref{thm:eg-var-skew-rad} in
Appendix~\ref{prf:eg-var-skew-rad}. In this example, the upper bound $
V^2 $ on $\var(\gradf(\vecw; \vecx, y))$ depends on dimension $d$ and the diameter of the parameter space. If the diameter is a constant, we have $V = \bigo(\sqrt{d})$. Moreover, the gradient skewness is bounded by a universal constant $S$ regardless of the size of the parameter space. In Appendix~\ref{eg:gaussian-var-skew}, we provide another example showing that when the features in $\vecx$ are i.i.d. Gaussian distributed, the coordinate-wise skewness can be upper bounded by $429$.

We now state our main technical results on the median-based algorithm, namely
statistical error guarantees for strongly convex, non-strongly convex, and smooth non-convex population loss functions $ F $. In the first two cases with a convex $F$, we assume that $\vecw^*$, the minimizer of $F(\cdot)$ in $\W$, is also the minimizer of $F(\cdot)$ in $\R^d$, i.e., $\gradF(\vecw^*)=0$.

\paragraph*{Strongly Convex Losses:}
We first consider the case where the population loss function
$F(\cdot)$ is strongly convex. Note that we do not require strong
convexity of the individual loss functions $f(\cdot;\vecz)$.
 
\begin{theorem}\label{thm:main-gd-sc}
Consider Option I in Algorithm~\ref{alg:robust-gd}. Suppose that Assumptions~\ref{asm:smoothness},~\ref{asm:bounded-variance}, and~\ref{asm:bounded-skewness} hold, $F(\cdot)$ is $\lambda_F$-strongly convex, and the fraction $\alpha$ of Byzantine machines satisfies
\begin{equation}\label{eq:alpha-m-n-condition}
\alpha + \sqrt{\frac{d\log(1+nm\widehat{L}D)}{m(1-\alpha)}} + 0.4748\frac{S}{\sqrt{n}} \le \frac{1}{2}-\epsilon
\end{equation}
for some $\epsilon>0$.
Choose step-size $\eta = 1/L_F$. Then, with probability at least $1-\frac{4d}{(1+nm\widehat{L}D)^d}$, after $T$ parallel iterations, we have 
$$
\twonms{\vecw^T - \vecw^*} \le (1-\frac{\lambda_F}{L_F + \lambda_F})^T\twonms{ \vecw^0 - \vecw^* } + \frac{2}{\lambda_F}\Delta,
$$
where
\begin{equation}\label{eq:def-error-delta}
\Delta := \bigo\Bigg( C_\epsilon V \Big( \frac{\alpha}{\sqrt{n}} + \sqrt{\frac{d\log(nm\widehat{L}D)}{nm}} + \frac{S}{n} \Big) \Bigg),
\end{equation}
and $C_\epsilon$ is defined as
\begin{equation}\label{eq:def-c-epsilon}
C_\epsilon := \sqrt{2\pi}\exp \Big(\frac{1}{2}(\Phi^{-1}(1-\epsilon))^2 \Big),
\end{equation}
with $\Phi^{-1}(\cdot)$ being the inverse of the cumulative distribution function of the standard Gaussian distribution $\Phi(\cdot)$.
\end{theorem}
We prove Theorem~\ref{thm:main-gd-sc} in Appendix~\ref{prf:main-gd-sc}. In~\eqref{eq:def-error-delta}, we hide universal constants and a higher order term that scales as $\frac{1}{nm}$, and the factor $C_\epsilon$ is a function of $\epsilon$; as a concrete example, $C_\epsilon \approx 4$ when $\epsilon = \frac{1}{6}$. Theorem~\ref{thm:main-gd-sc} together with the inequality $\log(1-x) \le -x$, guarantees that after running $
T \ge \frac{L_F + \lambda_F}{\lambda_F} \log( \frac{\lambda_F}{2\Delta}\twonms{\vecw^0 - \vecw^*} )$
parallel iterations, with high probability we can obtain a solution $\widehat{\vecw} = \vecw^T$ with error $\twonms{\widehat{\vecw} - \vecw^*}  \le \frac{4}{\lambda_F}\Delta$.

Here we achieve an error rate (defined as the distance between
$\widehat{\vecw}$ and the optimal solution $\vecw^*$) of the form
$\widetilde{\bigo}(\frac{\alpha}{\sqrt{n}} + \frac{1}{\sqrt{nm}} +
\frac{1}{n})$. In Section~\ref{sec:lower-bound}, we provide a lower
bound showing that the error rate of any algorithm is
$\widetilde{\Omega}(\frac{\alpha}{\sqrt{n}} + \frac{1}{\sqrt{nm}})$.
Therefore the first two terms in the upper bound cannot be improved.
The third term $\frac{1}{n}$ is due to the dependence of the median on the skewness of the gradients. When each worker machine has a sufficient amount of data, more specifically $n\gtrsim m$, we achieve an order-optimal error rate up to logarithmic factors.

\paragraph*{Non-strongly Convex  Losses:}
We next consider the case where the population risk function $F(\cdot)$ is convex, but not necessarily strongly convex. In this case, we need a mild technical assumption on the size of the parameter space $\W$. 

\begin{asm}[Size of $\W$]\label{asm:size-of-space}
The parameter space $\W$ contains the following $\ell_2$ ball centered at $\vecw^*$: $\{\vecw\in\R^d : \twonms{\vecw-\vecw^*} \le 2\twonms{\vecw^0 - \vecw^*} \}$.
\end{asm}
We then have the following result on the convergence rate in terms of the value of the population risk function.
\begin{theorem}\label{thm:main-gd-cvx}
Consider Option I in Algorithm~\ref{alg:robust-gd}. Suppose that Assumptions~\ref{asm:smoothness},~\ref{asm:bounded-variance},~\ref{asm:bounded-skewness} and~\ref{asm:size-of-space} hold, and that the population loss $F(\cdot)$ is convex, and $\alpha$ satisfies~\eqref{eq:alpha-m-n-condition} for some $\epsilon > 0$.
Define $\Delta$ as in~\eqref{eq:def-error-delta},
and choose step-size $\eta = 1/L_F$. Then, with probability at least $1-\frac{4d}{(1+nm\widehat{L}D)^d}$, after $T = \frac{L_F}{\Delta}\twonms{\vecw^0 - \vecw^*}$ parallel iterations, we have
$$
F(\vecw^T) - F(\vecw^*) \le 16 \twonms{\vecw^0 - \vecw^*} \Delta \Big(1 + \frac{1}{2L_F}\Delta \Big).
$$
\end{theorem}
We prove Theorem~\ref{thm:main-gd-cvx} in Appendix~\ref{prf:main-gd-cvx}. We observe that the error rate, defined as the excess risk $F(\vecw^T) - F(\vecw^*) $, again has the form $\widetilde{\bigo} \big(\frac{\alpha}{\sqrt{n}} + \frac{1}{\sqrt{nm}} + \frac{1}{n} \big)$.

\paragraph*{Non-convex  Losses:}
When $F(\cdot)$ is non-convex but smooth, we need a somewhat different technical assumption on the size of $\W$.
\begin{asm}[Size of $\W$]\label{asm:size-of-space-non-cvx}
Suppose that $\forall~\vecw\in\W$, $\twonms{\gradF(\vecw)}\le M$. We assume that $\W$ contains the $\ell_2$ ball $
\{\vecw\in\R^d : \twonms{\vecw - \vecw^0} \le \frac{2}{\Delta^2}(M+\Delta)(F(\vecw^0) - F(\vecw^*))\} $, where $\Delta$ is defined as in~\eqref{eq:def-error-delta}.
\end{asm}
We have the following guarantees on  the rate of convergence to a critical point of the population loss $F(\cdot)$.
\begin{theorem}\label{thm:main-gd-non-cvx}
Consider Option I in Algorithm~\ref{alg:robust-gd}. Suppose that Assumptions~\ref{asm:smoothness}~\ref{asm:bounded-variance},~\ref{asm:bounded-skewness} and~\ref{asm:size-of-space-non-cvx} hold, and  $\alpha$ satisfies~\eqref{eq:alpha-m-n-condition} for some $\epsilon > 0$.
Define $\Delta$ as in~\eqref{eq:def-error-delta},
and choose step-size $\eta =1/L_F$. With probability at least $1-\frac{4d}{(1+nm\widehat{L}D)^d}$, after $T = \frac{2L_F}{\Delta^2}(F(\vecw^0) - F(\vecw^*))$ parallel iterations, we have
$$
\min_{t=0,1,\ldots, T} \twonms{\gradF(\vecw^t)}  \le \sqrt{2}\Delta.
$$
\end{theorem}
We prove Theorem~\ref{thm:main-gd-non-cvx} in Appendix~\ref{prf:main-gd-non-cvx}. We again obtain an $\widetilde{\bigo}(\frac{\alpha}{\sqrt{n}} + \frac{1}{\sqrt{nm}} + \frac{1}{n})$ error rate in terms of the gap to a critical point of $F(\vecw)$.

\subsection{Guarantees for Trimmed-mean-based Gradient Descent}

We next analyze the robust distributed gradient descent algorithm based on coordinate-wise trimmed mean, namely Option II in Algorithm~\ref{alg:robust-gd}.
Here we need stronger assumptions on the tail behavior of the partial derivatives of the loss functions---in particular, sub-exponentiality. 

\begin{asm}[Sub-exponential gradients]\label{asm:sub-exponential}
We assume that for all $k\in[d]$ and $\vecw\in\W$, the partial derivative of $f(\vecw;\vecz)$ with respect to the $k$-th coordinate of $\vecw$,  $\partial_kf(\vecw;\vecz)$, is $v$-sub-exponential.
\end{asm}
The sub-exponential property implies that all the moments of the derivatives are bounded. This is a stronger assumption than the bounded absolute skewness (hence bounded third moments) required by the median-based GD algorithm. 

We use the same example as in Proposition~\ref{thm:eg-var-skew-rad} and show that the derivatives of the loss are indeed sub-exponential.
\begin{proposition}\label{thm:eg-sub-exp}
Consider the regression problem in Proposition~\ref{thm:eg-var-skew-rad}. For all $k\in[d]$ and $\vecw\in\W$, the partial derivative $\partial_kf(\vecw;\vecz)$ is $\sqrt{\sigma^2 + \twonms{\vecw - \vecw^*}^2}$-sub-exponential.
\end{proposition}
Proposition~\ref{thm:eg-sub-exp} is proved in Appendix~\ref{prf:eg-sub-exp}. We now proceed to establish the statistical guarantees of the trimmed-mean-based algorithm, for different loss function classes. When the population loss $F(\cdot)$ is convex, we again assume that the minimizer of $F(\cdot)$ in $\W$ is also its minimizer  in $\R^d$. The next three theorems are analogues of Theorems~\ref{thm:main-gd-sc}--\ref{thm:main-gd-non-cvx} for the median-based GD algorithm.

\paragraph*{Strongly Convex  Losses:}
We have the following result.
\begin{theorem}\label{thm:main-gd-sc-trim}
Consider Option II in Algorithm~\ref{alg:robust-gd}. Suppose that Assumptions~\ref{asm:smoothness} and~\ref{asm:sub-exponential} hold,  $F(\cdot)$ is $\lambda_F$-strongly convex, and $\alpha \le \beta \le \frac{1}{2} - \epsilon$ for some $\epsilon>0$.
Choose step-size $\eta = 1/L_F$. Then, with probability at least $1-\frac{4d}{(1+nm\widehat{L}D)^d}$, after $T$ parallel iterations, we have
$$
\twonms{\vecw^T - \vecw^*} \le \Big(1-\frac{\lambda_F}{L_F + \lambda_F}\Big)^T\twonms{ \vecw^0 - \vecw^* } + \frac{2}{\lambda_F}\Delta',
$$
where
\begin{equation}\label{eq:def-error-delta-trim}
\Delta' := \bigo \Big( \frac{vd}{\epsilon} \big( \frac{\beta }{\sqrt{n}} + \frac{1}{\sqrt{nm}}  \big)\sqrt{\log(nm\widehat{L}D)} \Big).
\end{equation}
\end{theorem}
We prove Theorem~\ref{thm:main-gd-sc-trim} in Appendix~\ref{prf:main-gd-sc-trim}. In~\eqref{eq:def-error-delta-trim}, we hide universal constants and higher order terms that scale as $\frac{\beta}{n}$ or $\frac{1}{nm}$. By running $T \ge \frac{L_F + \lambda_F}{\lambda_F} \log( \frac{\lambda_F}{2\Delta'}\twonms{\vecw^0 - \vecw^*} )$ parallel iterations, we can obtain a solution $\widehat{\vecw} = \vecw^T$ satisfying $\twonms{\widehat{\vecw} - \vecw^*} \le \widetilde{\bigo}(\frac{\beta}{\sqrt{n}} + \frac{1}{\sqrt{nm}}) $. Note that one needs to choose the parameter for trimmed mean to satisfy $\beta \ge \alpha$. If we set $\beta = c\alpha$ for some universal constant $c\ge 1$, we can achieve an order-optimal error rate $\widetilde{\bigo}(\frac{\alpha}{\sqrt{n}} + \frac{1}{\sqrt{nm}})$.

\paragraph*{Non-strongly Convex  Losses:}
Again imposing Assumption~\ref{asm:size-of-space} on the size of $\W$, we have the following guarantee.
\begin{theorem}\label{thm:main-gd-cvx-trim}
Consider Option II in Algorithm~\ref{alg:robust-gd}. Suppose that Assumptions~\ref{asm:smoothness},~\ref{asm:size-of-space} and~\ref{asm:sub-exponential} hold,  $F(\cdot)$ is convex, and $\alpha \le \beta \le \frac{1}{2} - \epsilon$ for some $\epsilon>0$.
Choose step-size $\eta = 1/L_F$, and define $\Delta'$ as in~\eqref{eq:def-error-delta-trim}. Then, with probability at least $1-\frac{4d}{(1+nm\widehat{L}D)^d}$, after $T = \frac{L_F}{\Delta'}\twonms{\vecw^0 - \vecw^*}$ parallel iterations, we have
$$
F(\vecw^T) - F(\vecw^*) \le 16 \twonms{\vecw^0 - \vecw^*} \Delta' \Big(1 + \frac{1}{2L_F}\Delta' \Big).
$$
\end{theorem}

The proof of Theorem~\ref{thm:main-gd-cvx-trim} is similar to that of Theorem~\ref{thm:main-gd-cvx}, and we refer readers to Remark~\ref{rmk:non-strongly-non-cvx} in Appendix~\ref{prf:main-gd-sc-trim}. Again, by choosing $\beta = c\alpha$ $(c\ge 1)$, we obtain the $\widetilde{\bigo}(\frac{\alpha}{\sqrt{n}} + \frac{1}{\sqrt{nm}})$ error rate in the function value of $F(\vecw)$.

\paragraph*{Non-convex  Losses:}
In this case, imposing a version of Assumption~\ref{asm:size-of-space-non-cvx} on the size of $ \W $, we have the following.

\begin{theorem}\label{thm:main-gd-non-cvx-trim}
Consider Option II in Algorithm~\ref{alg:robust-gd}, and define $\Delta'$ as in~\eqref{eq:def-error-delta-trim}. Suppose that Assumptions~\ref{asm:smoothness} and~\ref{asm:sub-exponential} hold,  Assumption~\ref{asm:size-of-space-non-cvx} holds with $ \Delta $ replaced by $ \Delta' $, and $\alpha \le \beta \le \frac{1}{2} - \epsilon$ for some $\epsilon>0$.
Choose step-size $\eta = 1/L_F$. Then, with probability at least $1-\frac{4d}{(1+nm\widehat{L}D)^d}$, after $T = \frac{2L_F}{\Delta'^2}(F(\vecw^0) - F(\vecw^*))$ parallel iterations, we have
$$
\min_{t=0,1,\ldots, T} \twonms{\gradF(\vecw^t)}  \le \sqrt{2}\Delta'.
$$
\end{theorem}
The proof of Theorem~\ref{thm:main-gd-non-cvx-trim} is similar to that of Theorem~\ref{thm:main-gd-non-cvx}; see Remark~\ref{rmk:non-strongly-non-cvx} in Appendix~\ref{prf:main-gd-sc-trim}. By choosing $\beta = c\alpha$ with $c\ge 1$, we again achieve the statistical rate $\widetilde{\bigo}(\frac{\alpha}{\sqrt{n}} + \frac{1}{\sqrt{nm}})$. 

\subsection{Comparisons}

We compare the performance guarantees of the above two robust distribute GD algorithms. The trimmed-mean-based algorithm achieves the statistical error rate $\widetilde{\bigo}(\frac{\alpha}{\sqrt{n}} + \frac{1}{\sqrt{nm}})$, which is order-optimal for strongly convex loss. In comparison, the rate of the median-based algorithm is $\widetilde{\bigo}(\frac{\alpha}{\sqrt{n}} + \frac{1}{\sqrt{nm}} + \frac{1}{n})$, which has an additional $ \frac{1}{n} $ term and is only optimal when $n \gtrsim m$. In particular, the trimmed-mean-based algorithm has better rates when each worker machine has small local sample size---the rates are meaningful even in the extreme case $ n=\bigo(1) $. On the other hand, the median-based algorithm requires milder tail/moment assumptions  on the loss derivatives (bounded skewness) than its trimmed-mean counterpart (sub-exponentiality). Finally, the trimmed-mean operation requires an additional parameter $\beta$, which can be any upper bound on the fraction $ \alpha $ of Byzantine machines in order to guarantee robustness. Using an overly large $\beta$ may lead to a looser bound and sub-optimal performance. In contrast, median-based GD does not require knowledge of $\alpha$. We summarize these observations in Table~\ref{tab:comparison}. We see that the two algorithms are complementary to each other, and our experiment results corroborate this point.  
 
\begin{table}[htbp]
\centering
\begin{tabular}{|c|c|c|}
\hline
        & median GD & trimmed mean GD \\  \hline
Statistical error rate & $\widetilde{\bigo}(\frac{\alpha}{\sqrt{n}} + \frac{1}{\sqrt{nm}} + \frac{1}{n})$ & $\widetilde{\bigo}(\frac{\alpha}{\sqrt{n}} + \frac{1}{\sqrt{nm}})$  \\ \hline
Distribution of $\partial_kf(\vecw;\vecz)$ & Bounded skewness & Sub-exponential \\ \hline
$ \alpha $ known? & No & Yes \\ \hline
\end{tabular}
\caption{Comparison between the two robust distributed gradient descent algorithms.}
\label{tab:comparison}
\end{table}

\section{Robust One-round Algorithm}\label{sec:com-efficient}

As mentioned, in our distributed computing framework, the communication cost is proportional to the number of parallel iterations. The above two GD algorithms both require a number iterations depending on the desired accuracy. Can we further reduce the communication cost while keeping the algorithm Byzantine-robust and statistically optimal?

A natural candidate is the so-called one-round algorithm. Previous work has considered a standard one-round scheme where each local machine computes the empirical risk minimizer (ERM) using its local data and the master machine receives all workers' ERMs and computes their \emph{average}~\citep{zhang2012communication}. Clearly, a single Byzantine machine can arbitrary skew the output of this algorithm. We instead consider a Byzantine-robust one-round algorithm. As detailed in Algorithm~\ref{alg:robust-one-round}, we employ the coordinate-wise median operation to aggregate all the ERMs.
\begin{algorithm}[h]
	\caption{Robust One-round Algorithm}\label{alg:robust-one-round}
	\begin{algorithmic}
		\PARFOR{$i\in[m]$}
		\STATE \textit{\underline{Worker machine $i$}}: compute:
		\STATE 
		$$
		\widehat{\vecw}^i \leftarrow \begin{cases}
		\arg\min_{\vecw\in\W} F_i(\vecw) & \text{normal worker machines} \\
		* & \text{Byzantine machines} \end{cases}
		$$
		\STATE send $\widehat{\vecw}^i$ to master machine.
		\ENDPARFOR
		\STATE \textit{\underline{Master machine}}: compute $\widehat{\vecw} \leftarrow \med\{\widehat{\vecw}^i : i\in[m]\}$.
	\end{algorithmic}
\end{algorithm}

Our main result is a characterization of the error rate of Algorithm~\ref{alg:robust-one-round} in the presence of Byzantine failures. We are only able to establish such a guarantee when the loss functions are quadratic and $\W=\R^d$. However, one can implement this algorithm in problems with other loss functions.

\begin{definition}[Quadratic loss function]\label{def:quadratic}
The loss function $f(\vecw;\vecz)$ is quadratic if it can be written as
$$
f(\vecw;\vecz) = \frac{1}{2}\vecw\tsp\matH\vecw + \vecp\tsp\vecw + c,
$$
where $ \vecz =(\matH, \vecp, c) $, $\matH$, and $\vecp$, and $c$ are drawn from the distributions $\D_H$, $\D_p$, and $\D_c$, respectively.
\end{definition}

Denote by $\matH_F$, $\vecp_F$, and $c_F$ the expectations of $\matH$, $\vecp$, and $c$, respectively. Thus the population risk function takes the form $F(\vecw) = \frac{1}{2}\vecw\tsp\matH_F\vecw + \vecp_F\tsp \vecw + c_F$.

We need a technical assumption which guarantees that each normal worker machine has unique ERM.
\begin{asm}[Strong convexity of $F_i$]\label{asm:unique}
With probability $1$, the empirical risk minimization function $F_i(\cdot)$ on each normal machine  is strongly convex.
\end{asm}
Note that this assumption is imposed on $ F_i (\vecw)$, rather than on the individual loss $ f(\vecw; \vecz) $ associated with a single data point.
This assumption is satisfied, for example, when all $f(\cdot;\vecz)$'s are strongly convex, or in the linear regression problems with the features $ \vecx $ drawn from some continuous distribution (e.g. isotropic Gaussian) and $n\ge d$.
We have the following guarantee for the robust one-round algorithm.
\begin{theorem}\label{thm:one-round}
Suppose that $\forall~\vecz\in\Z$, the loss function $f(\cdot; \vecz)$ is convex and quadratic, $F(\cdot)$ is $\lambda_F$-strongly convex, and Assumption~\ref{asm:unique} holds. Assume that $\alpha$ satisfies
$$
\alpha + \sqrt{\frac{\log(nmd)}{2m(1-\alpha)}} + \frac{\widetilde{C}}{\sqrt{n}} \le \frac{1}{2} - \epsilon
$$
for some $\epsilon>0$, where $\widetilde{C}$ is a quantity that depends on $\D_H$,  $\D_p$, $\lambda_F$ and is monotonically decreasing in $n$. Then, with probability at least $1-\frac{4}{nm}$, the output $ \widehat{\vecw} $ of the robust one-round algorithm satisfies
$$
\twonms{\widehat{\vecw} - \vecw^*} \le  \frac{C_\epsilon}{\sqrt{n}}\widetilde{\sigma}\big( \alpha + \sqrt{\frac{\log(nmd)}{2m(1-\alpha)}} + \frac{\widetilde{C}}{\sqrt{n}}\big),
$$
where $C_\epsilon$ is defined as in~\eqref{eq:def-c-epsilon} and
$$
\widetilde{\sigma}^2 :=  \EE \big[ \twonms{\matH_F^{-1}\big(  (\matH - \matH_F ) \matH_F^{-1}\vecp_F-(\vecp - \vecp_F)  \big)}^2\big],
$$
with $\matH$ and $\vecp$ drawn from $\D_H$ and $\D_p$, respectively.
\end{theorem}
We prove Theorem~\ref{thm:one-round} and provide an explicit expression of $\widetilde{C}$ in Appendix~\ref{prf:one-round}. In terms of the dependence on $\alpha$, $n$, and $m$, the robust one-round algorithm achieves the same error rate as the robust gradient descent algorithm based on coordinate-wise median, i.e., $\widetilde{\bigo}(\frac{\alpha}{\sqrt{n}} + \frac{1}{\sqrt{nm}} + \frac{1}{n})$, for quadratic problems. Again, this rate is optimal when $n\gtrsim m$. Therefore, at least for quadratic loss functions, the robust one-round algorithm has similar theoretical performance as the robust gradient descent algorithm with significantly less communication cost. Our experiments show that the one-round algorithm has good empirical performance for other losses as well. 

\section{Lower Bound}\label{sec:lower-bound}

In this section, we provide a lower bound on the error rate for strongly convex losses, which implies that the $\frac{\alpha}{\sqrt{n}} + \frac{1}{\sqrt{nm}}$ term is unimprovable. This lower bound is derived using a mean estimation problem, and is an extension of the lower bounds in the robust mean estimation literature such as~\citet{chen2015robust,lai2016agnostic}.

We consider the problem of estimating the mean $\vecmu$ of some random variable $\vecz\sim\Z$, which is equivalent to solving the following minimization problem:
\begin{equation}\label{eq:mean-estimation}
\vecmu = \arg\min_{\vecw\in\W} \EE_{\vecz\sim\Z}[\twonms{\vecw - \vecz}^2],
\end{equation}
Note that this is a special case of the general learning problem~\eqref{eq:min-loss}. We consider the same distributed setting as in Section~\ref{sec:gradient-descent}, with a minor technical difference regarding the Byzantine machines. We assume that each of the $ m $ worker machines is Byzantine with probability $\alpha$, independently of each other. The parameter $\alpha$ is therefore the \emph{expected} fraction of Byzantine machines. This setting makes the analysis slightly easier, and we believe the result can be extended to the original setting. 

In this setting we have the following lower bound.
\begin{observation}\label{obs:lower-bound}
Consider the distributed mean estimation problem in~\eqref{eq:mean-estimation} with Byzantine failure probability $\alpha$, and suppose that $\Z$ is Gaussian distribution with mean $\vecmu$ and covariance matrix $\sigma^2\mat{I}$ $(\sigma = \bigo(1))$. Then, any algorithm that computes an estimation $\widehat{\vecmu}$ of the mean from the data has a constant probability of error $\twonms{\widehat{\vecmu} - \vecmu} = \Omega (\frac{\alpha}{\sqrt{n}} + \sqrt{\frac{d}{nm}})$.
\end{observation}
We prove Observation~\ref{obs:lower-bound} in Appendix~\ref{prf:lower-bound}. According to this observation, we see that the $\frac{\alpha}{\sqrt{n}} +\frac{1}{\sqrt{nm}}$ dependence cannot be avoided, which in turn implies the order-optimality of the results in Theorem~\ref{thm:main-gd-sc} (when $n\gtrsim m$) and Theorem~\ref{thm:main-gd-sc-trim}.

\section{Experiments}\label{sec:experiments}

We conduct experiments to show the effectiveness of the median and trimmed mean operations. Our experiments are implemented with Tensorflow~\citep{abadi2016tensorflow} on Microsoft Azure system. We use the MNIST~\citep{lecun1998gradient} dataset and randomly partition the 60,000 training data into $m$ subsamples with equal sizes. We use these subsamples to represent the data on $m$ machines.

In the first experiment, we compare the performance of distributed
gradient descent algorithms in the following four settings: 1) $\alpha = 0$ (no Byzantine machines), using vanilla distributed gradient descent (aggregating the gradients by taking the mean), 2) $\alpha > 0$, using vanilla distributed gradient descent, 3) $\alpha > 0$, using median-based algorithm, and 4) $\alpha > 0$, using trimmed-mean-based algorithm. We generate the Byzantine machines in the following way: we replace every training label $y$ on these machines with $9-y$, e.g., $0$ is replaced with $9$, $1$ is replaced with $8$, etc, and the Byzantine machines simply compute gradients based on these data. We also note that when generating the Byzantine machines, we do not simply add extreme values in the features or gradients; instead, the Byzantine machines send messages to the master machine with moderate values.
 
We train a multi-class logistic regression model and a convolutional
neural network model using distributed gradient descent, and for each
model, we compare the test accuracies in the aforementioned four settings. For the convolutional neural network model, we use the stochastic version of the distributed gradient descent algorithm; more specifically, in every iteration, each worker machine computes the gradient using $10\%$ of its local data. We periodically check the test errors, and the convergence performances are shown in Figure~\ref{fig:convergence}. The final test accuracies are presented in Tables~\ref{tab:logistic} and~\ref{tab:cnn}.

\begin{table}[htbp]
\centering
\begin{tabular}{|c|c|c|c|c|}
\hline
$\alpha$ & 0 & \multicolumn{3}{c|}{0.05}  \\ \hline
Algorithm & mean & mean & median & trimmed mean \\  \hline
Test accuracy (\%) & 88.0 & 76.8  &  87.2  & 86.9 \\ \hline
\end{tabular}
\caption{Test accuracy on the logistic regression model using gradient descent. We set $m=40$, and for trimmed mean, we choose $\beta = 0.05$.}
\label{tab:logistic}
\end{table}

\begin{table}[htbp]
\centering
\begin{tabular}{|c|c|c|c|c|}
\hline
$\alpha$ & 0 & \multicolumn{3}{c|}{0.1}  \\ \hline
Algorithm & mean & mean & median & trimmed mean \\  \hline
Test accuracy (\%) & 94.3 & 77.3  &  87.4  & 90.7 \\ \hline
\end{tabular}
\caption{Test accuracy on the convolutional neural network model using gradient descent. We set $m=10$, and for trimmed mean, we choose $\beta = 0.1$.}
\label{tab:cnn}
\end{table}

\begin{figure}[h]
\centering 
\subfigure{\label{fig:softmax}\includegraphics[width=0.7\linewidth]{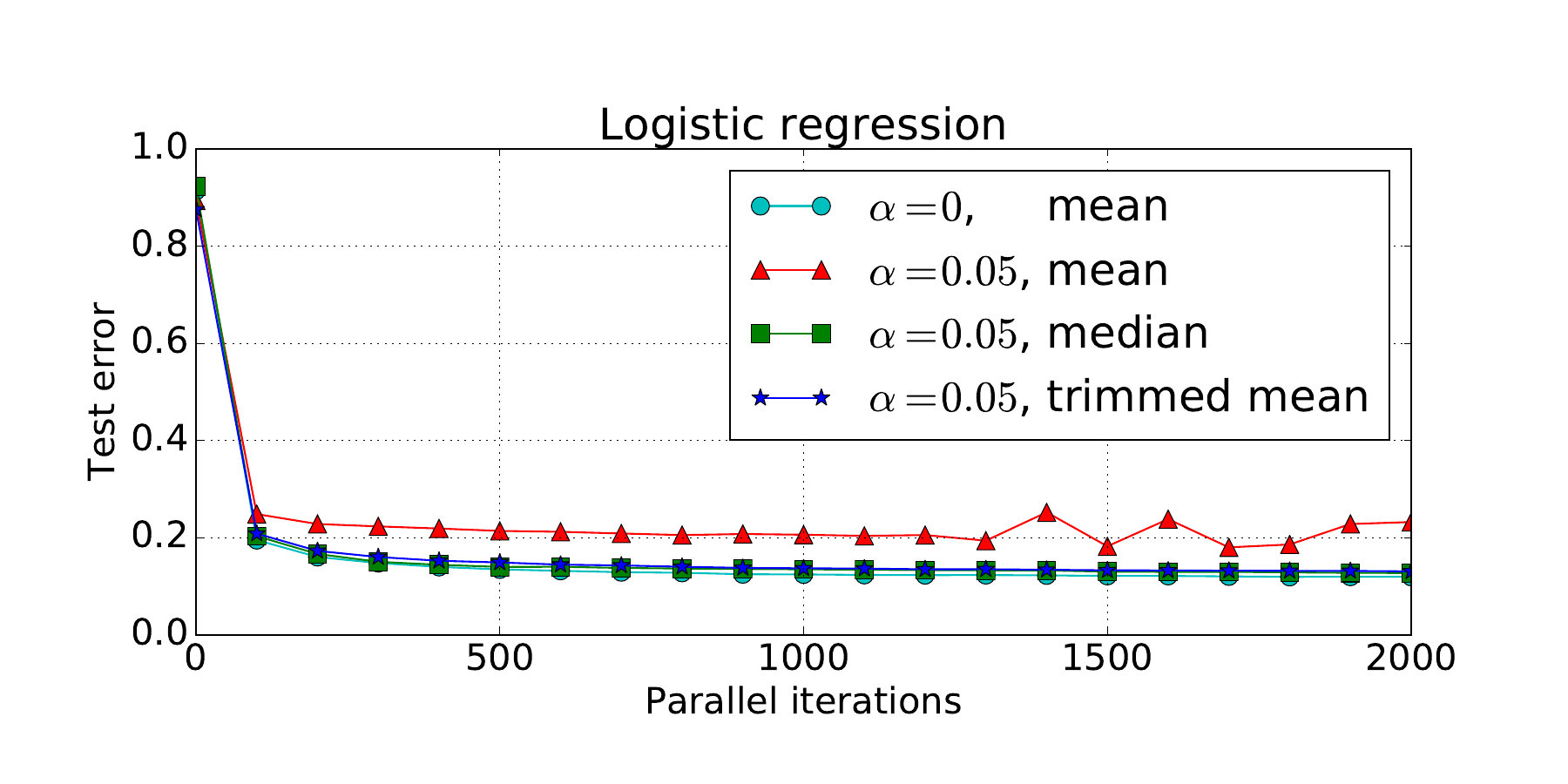}}\\
\subfigure{\label{fig:cnn}\includegraphics[width=0.7\linewidth]{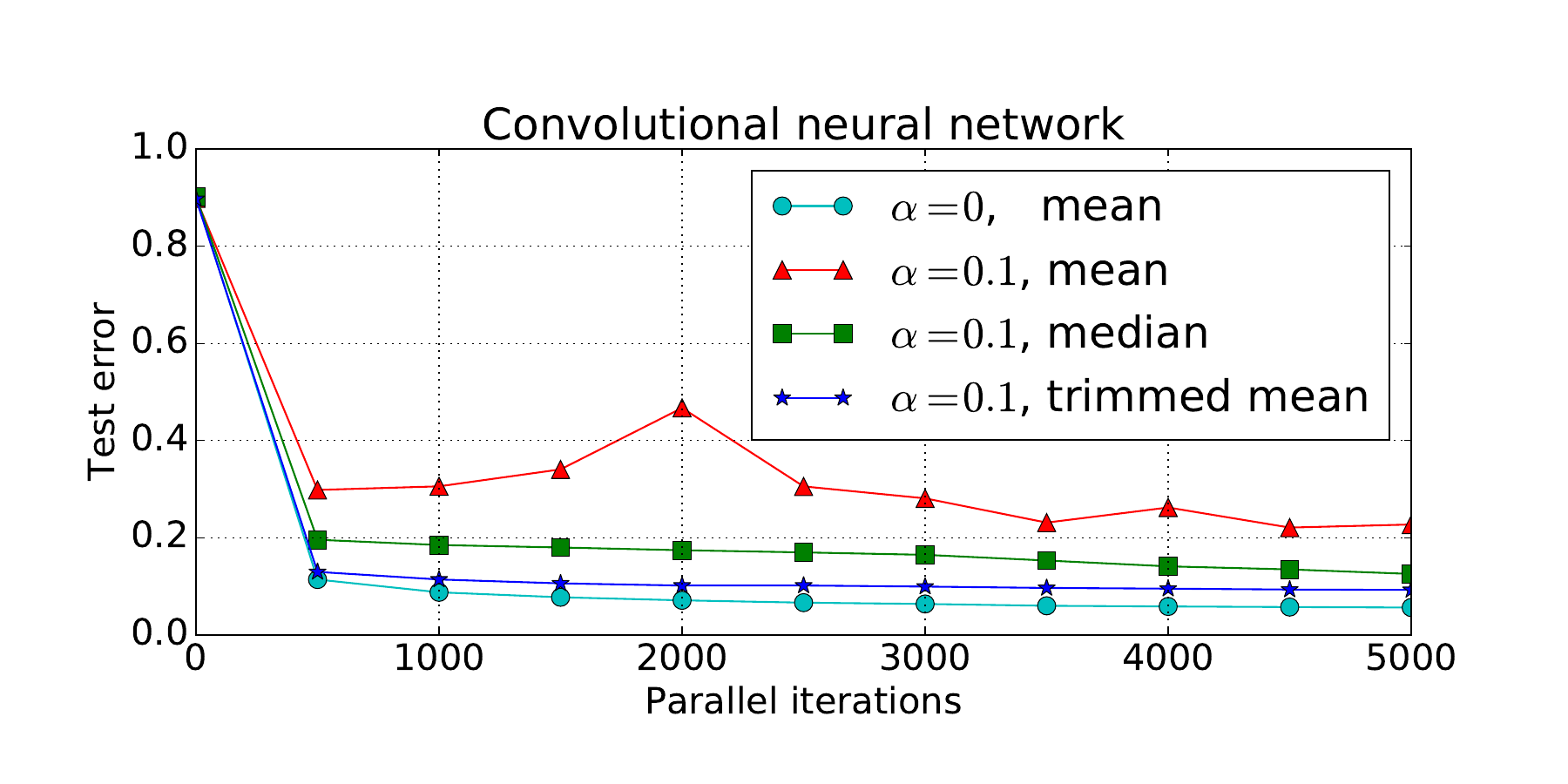}}
\caption{Test error vs the number of parallel iterations.}
\label{fig:convergence}
\end{figure}

As we can see, in the adversarial settings, the vanilla distributed
gradient descent algorithm suffers from severe performance loss, and
using the median and trimmed mean operations, we observe significant
improvement in test accuracy. This shows these two operations can indeed
defend against Byzantine failures.

In the second experiment, we compare the performance of distributed
one-round algorithms in the following three settings: 1) $\alpha = 0$, mean aggregation, 2) $\alpha > 0$, mean aggregation, and 3) $\alpha > 0$, median aggregation. In this experiment, the training labels on the Byzantine machines are i.i.d. uniformly sampled from $\{0,\ldots,9\}$, and these machines train models using the faulty data. We choose the multi-class logistic regression model, and the test accuracies are presented in Table~\ref{tab:one_round}.

\begin{table}[htbp]
\centering
\begin{tabular}{|c|c|c|c|}
\hline
$\alpha$ & 0 & \multicolumn{2}{c|}{0.1}  \\ \hline
Algorithm & mean & mean & median \\  \hline
Test accuracy (\%) & 91.8 & 83.7  &  89.0  \\ \hline
\end{tabular}
\caption{Test accuracy on the logistic regression model using one-round algorithm. We set $m=10$.}
\label{tab:one_round}
\end{table}

As we can see, for the one-round algorithm, although the theoretical guarantee is only proved for quadratic loss, in practice, the median-based one-round algorithm still improves the test accuracy in problems with other loss functions, such as the logistic loss here.

\section{Conclusions}\label{sec:conclusion}
In this paper, we study Byzantine-robust distributed statistical
learning algorithms with a focus on statistical optimality. We analyze two robust distributed gradient descent algorithms --- one is based on coordinate-wise median and the other is based on coordinate-wise trimmed mean. We show that the trimmed-mean-based algorithm can achieve order-optimal $\widetilde{\bigo}(\frac{\alpha}{\sqrt{n}} + \frac{1}{\sqrt{nm}})$ error rate, whereas the median-based algorithm can achieve $\widetilde{\bigo}(\frac{\alpha}{\sqrt{n}} + \frac{1}{\sqrt{nm}} + \frac{1}{n})$ under weaker assumptions. We further study learning algorithms that have better communication efficiency. We propose a simple one-round algorithm that aggregates local solutions using coordinate-wise median. We show that for strongly convex quadratic problems, this algorithm can achieve $\widetilde{\bigo}(\frac{\alpha}{\sqrt{n}} + \frac{1}{\sqrt{nm}} + \frac{1}{n})$ error rate, similar to the median-based gradient descent algorithm. Our experiments validates the effectiveness of the median and trimmed mean operations in the adversarial setting.

\subsection*{Acknowledgements}
D. Yin is partially supported by Berkeley DeepDrive Industry Consortium. Y. Chen is partially supported by NSF CRII award 1657420 and grant 1704828. K. Ramchandran is partially supported by NSF CIF award 1703678 and Gift award from Huawei. P. Bartlett is partially supported by NSF grant IIS-1619362. Cloud computing resources are provided by a Microsoft Azure for Research award.

\bibliographystyle{abbrvnat}
\bibliography{ref}

\appendix
\section*{Appendix}
\section{Variance, Skewness, and Sub-exponential Property}

\subsection{Proof of Proposition~\ref{thm:eg-var-skew-rad}}\label{prf:eg-var-skew-rad}
We use the simplified notation $f(\vecw) := f(\vecw; \vecx, y)$. One can directly compute the gradients:
$$
\gradf(\vecw) = \vecx(\vecx\tsp\vecw - y) = \vecx\vecx\tsp(\vecw - \vecw^*) - \xi\vecx,
$$
and thus
$$
\gradF(\vecw) = \EXPS{\gradf(\vecw)} = \vecw - \vecw^*.
$$
Define $\Delta(\vecw) := \gradf(\vecw) - \gradF(\vecw)$ with its $k$-th element being $\Delta_k(\vecw)$. We now compute the variance and absolute skewness of $\Delta_k(\vecw)$.

We can see that
\begin{equation}\label{eq:def-delta-k-rad}
\Delta_k(\vecw) = \sum_{\substack{1\le i \le d \\ i\neq k}} x_kx_i(w_i - w_i^*) + (x_k^2-1)(w_k - w_k^*) - \xi x_k.
\end{equation}
Thus,
\begin{equation}\label{eq:var-delta-k-2-rad}
\EXPS{\Delta_k^2(\vecw)} = \EXPS{\sum_{\substack{1\le i \le d \\ i\neq k}} x_k^2x_i^2(w_i - w_i^*)^2  + \xi^2 x_k^2} =\twonms{\vecw - \vecw^*}^2  - (w_k - w_k^*)^2 + \sigma^2, 
\end{equation}
which yields
$$
\var(\gradf(\vecw)) = \EXPS{\twonms{\gradf(\vecw) - \gradF(\vecw)}^2} = (d-1) \twonms{\vecw - \vecw^*}^2 + d\sigma^2.
$$
Then we proceed to bound $\gamma(\Delta_k(\vecw))$. By Jensen's inequality, we know that
\begin{equation}\label{eq:bound-skewness-jensen-rad}
\gamma(\Delta_k(\vecw)) = \frac{\EXPS{|\Delta_k(\vecw)|^3}}{\var(\Delta_k(\vecw))^{3/2}} \le \sqrt{\frac{\EXPS{\Delta_k^6(\vecw)}}{\var(\Delta_k(\vecw))^3}}
\end{equation}
We first find a lower bound for $\var(\Delta_k(\vecw))^3$. According to~\eqref{eq:var-delta-k-2-rad}, we know that
$$
\var(\Delta_k(\vecw))^3 = \big( \sum_{\substack{1\le i \le d \\ i\neq k}} (w_i - w_i^*)^2  + \sigma^2 \big)^3 \ge \big( \sum_{\substack{1\le i \le d \\ i\neq k}} (w_i - w_i^*)^2 \big)^3 + \sigma^6.
$$
Define the following three quantities.
\begin{align}
W_1 &= \sum_{\substack{1\le i \le d \\ i\neq k}} (w_i - w_i^*)^6 \label{eq:def-w1-rad} \\
W_2 &= \sum_{\substack{1\le i, j \le d \\ i,j\neq k \\ i\neq j}} (w_i - w_i^*)^4(w_j - w_j^*)^2 \label{eq:def-w2-rad} \\
W_3 &= \sum_{\substack{1\le i, j, \ell \le d \\ i,j,\ell\neq k \\ i\neq j, i\neq \ell, j\neq \ell}} (w_i - w_i^*)^2(w_j - w_j^*)^2(w_\ell - w_\ell^*)^2 \label{eq:def-w3-rad}
\end{align}
By simple algebra, one can check that
\begin{equation}\label{eq:expansion}
\big( \sum_{\substack{1\le i \le d \\ i\neq k}} (w_i - w_i^*)^2 \big)^3 = W_1 + 3W_2 + W_3,
\end{equation}
and thus
\begin{equation}\label{eq:var3-lb-rad}
\var(\Delta_k(\vecw))^3 \ge W_1 + 3W_2 + W_3 + \sigma^6.
\end{equation}
Then, we find an upper bound on $\EXPS{\Delta_k^6(\vecw)}$. According to~\eqref{eq:def-delta-k-rad}, and H\"{o}lder's inequality, we know that
\begin{align}
\EXPS{\Delta_k^6(\vecw)} &= \EXPS{(\sum_{\substack{1\le i \le d \\ i\neq k}} x_kx_i(w_i - w_i^*)  - \xi x_k)^6} \le 32 \big( \EXPS{(\sum_{\substack{1\le i \le d \\ i\neq k}} x_kx_i(w_i - w_i^*))^6} +  \EXPS{\xi^6 x_k^6} \big) \nonumber \\
&=32 \big( \EXPS{(\sum_{\substack{1\le i \le d \\ i\neq k}}x_i(w_i - w_i^*))^6} + 15\sigma^6 \big) \label{eq:bound-var-6-rad},
\end{align}
where in the last inequality we use the moments of Gaussian random variables. Then, we compute the first term in~\eqref{eq:bound-var-6-rad}. By algebra, one can obtain
\begin{align}
\EXPS{(\sum_{\substack{1\le i \le d \\ i\neq k}}x_i(w_i - w_i^*))^6} =& \EXPS{\sum_{\substack{1\le i \le d \\ i\neq k}} x_i^6(w_i - w_i^*)^6} + 15 \EXPS{ \sum_{\substack{1\le i, j \le d \\ i,j\neq k \\ i\neq j}} x_i^4x_j^2(w_i-w_i^*)^4(w_j - w_j^*)^2 } \nonumber\\
&+ 15 \EXPS{ \sum_{\substack{1\le i, j, \ell \le d \\ i,j,\ell\neq k \\ i\neq j, i\neq \ell, j\neq \ell}} x_i^2x_j^2x_\ell^2(w_i - w_i^*)^2(w_j - w_j^*)^2(w_\ell - w_\ell^*)^2 } \nonumber\\
=& W_1 + 15W_2 + 15 W_3. \label{eq:bound-var-6-2-rad} 
\end{align}
Combining~\eqref{eq:bound-var-6-rad} and~\eqref{eq:bound-var-6-2-rad}, we get
\begin{equation}\label{eq:bound-var-6-3-rad}
\EXPS{\Delta_k^6(\vecw)} \le 32( W_1 + 15W_2 + 15W_3 + 15\sigma^6).
\end{equation}
Combining~\eqref{eq:var3-lb-rad} and~\eqref{eq:bound-var-6-3-rad}, we get
$$
\gamma(\Delta_k(\vecw)) \le \sqrt{\frac{\EXPS{\Delta_k^6(\vecw)}}{\var(\Delta_k(\vecw))^3}} \le \sqrt{\frac{32(W_1 + 15W_2 + 15W_3 + 15\sigma^6)}{W_1 + 3W_2 + W_3  + \sigma^6}} \le 480.
$$

\subsection{Example of Regression with Gaussian Features}\label{eg:gaussian-var-skew}
\begin{proposition}\label{thm:eg-var-skew}
Suppose that each data point consists of a feature $\vecx\in\R^d$ and a label $y\in\R$, and the label is generated by 
$$
y = \vecx\tsp\vecw^* + \xi
$$
with some $\vecw^*\in\W$. Assume that the elements of $\vecx$ are i.i.d. samples of standard Gaussian distribution, and that the noise $\xi$ is independent of $\vecx$ and drawn from Gaussian distribution $\mathcal{N}(0, \sigma^2)$. Define the quadratic loss function $f(\vecw; \vecx, y) = \frac{1}{2}(y - \vecx\tsp\vecw)^2$. Then, we have
$$
\var(\gradf(\vecw; \vecx, y)) = (d+1) \twonms{\vecw - \vecw^*}^2 + d\sigma^2,
$$
and
$$
\|\gamma(\gradf(\vecw;\vecx, y))\|_\infty \le 429.
$$
\end{proposition}

\begin{proof}
We use the same simplified notation as in Appendix~\ref{prf:eg-var-skew-rad}. One can also see that~\eqref{eq:def-delta-k-rad} still holds for in the Gaussian setting.
Thus,
\begin{align}
\EXPS{\Delta_k^2(\vecw)} &= \EXPS{\sum_{\substack{1\le i \le d \\ i\neq k}} x_k^2x_i^2(w_i - w_i^*)^2 + (x_k^2-1)^2(w_k - w_k^*)^2 + \xi^2 x_k^2} \nonumber \\
&= \sum_{\substack{1\le i \le d \\ i\neq k}} (w_i - w_i^*)^2 + 2(w_k - w_k^*)^2 + \sigma^2 \label{eq:var-delta-k-1}  \\
&=\twonms{\vecw - \vecw^*}^2 + (w_k - w_k^*)^2 + \sigma^2, \label{eq:var-delta-k-2}
\end{align}
which yields
$$
\var(\gradf(\vecw)) = \EXPS{\twonms{\gradf(\vecw) - \gradF(\vecw)}^2} = (d+1) \twonms{\vecw - \vecw^*}^2 + d\sigma^2.
$$
Then we proceed to bound $\gamma(\Delta_k(\vecw))$. By Jensen's inequality, we know that
\begin{equation}\label{eq:bound-skewness-jensen}
\gamma(\Delta_k(\vecw)) = \frac{\EXPS{|\Delta_k(\vecw)|^3}}{\var(\Delta_k(\vecw))^{3/2}} \le \sqrt{\frac{\EXPS{\Delta_k^6(\vecw)}}{\var(\Delta_k(\vecw))^3}}
\end{equation}
We first find a lower bound for $\var(\Delta_k(\vecw))^3$. According to~\eqref{eq:var-delta-k-1}, we know that
\begin{align*}
\var(\Delta_k(\vecw))^3 &= \big( \sum_{\substack{1\le i \le d \\ i\neq k}} (w_i - w_i^*)^2 + 2(w_k - w_k^*)^2 + \sigma^2 \big)^3\\
& \ge \big( \sum_{\substack{1\le i \le d \\ i\neq k}} (w_i - w_i^*)^2 \big)^3 + 8(w_k - w_k^*)^6 + \sigma^6.
\end{align*}
Define the $W_1$, $W_2$, and $W_3$ as in~\eqref{eq:def-w1-rad},~\eqref{eq:def-w2-rad}, and~\eqref{eq:def-w3-rad}. We can also see that~\eqref{eq:expansion} still holds, and thus
\begin{equation}\label{eq:var3-lb}
\var(\Delta_k(\vecw))^3 \ge W_1 + 3W_2 + W_3 + 8(w_k - w_k^*)^6 + \sigma^6.
\end{equation}
Then, we find an upper bound on $\EXPS{\Delta_k^6(\vecw)}$. According to~\eqref{eq:def-delta-k-rad}, and H{\"{o}}lder's inequality, we know that
\begin{align}
\EXPS{\Delta_k^6(\vecw)} &= \EXPS{(\sum_{\substack{1\le i \le d \\ i\neq k}} x_kx_i(w_i - w_i^*) + (x_k^2-1)(w_k - w_k^*) - \xi x_k)^6} \nonumber \\
&\le 243 \big( \EXPS{(\sum_{\substack{1\le i \le d \\ i\neq k}} x_kx_i(w_i - w_i^*))^6} + \EXPS{(x_k^2-1)^6(w_k - w_k^*)^6} + \EXPS{\xi^6 x_k^6} \big) \nonumber \\
&=243 \big( 15\EXPS{(\sum_{\substack{1\le i \le d \\ i\neq k}}x_i(w_i - w_i^*))^6} + 6040 (w_k - w_k^*)^6 + 225\sigma^6 \big) \label{eq:bound-var-6},
\end{align}
where in the last inequality we use the moments of Gaussian random variables. Then, we compute the first term in~\eqref{eq:bound-var-6}. By algebra, one can obtain
\begin{align}
\EXPS{(\sum_{\substack{1\le i \le d \\ i\neq k}}x_i(w_i - w_i^*))^6} =& \EXPS{\sum_{\substack{1\le i \le d \\ i\neq k}} x_i^6(w_i - w_i^*)^6} + 15 \EXPS{ \sum_{\substack{1\le i, j \le d \\ i,j\neq k \\ i\neq j}} x_i^4x_j^2(w_i-w_i^*)^4(w_j - w_j^*)^2 } \nonumber\\
&+ 15 \EXPS{ \sum_{\substack{1\le i, j, \ell \le d \\ i,j,\ell\neq k \\ i\neq j, i\neq \ell, j\neq \ell}} x_i^2x_j^2x_\ell^2(w_i - w_i^*)^2(w_j - w_j^*)^2(w_\ell - w_\ell^*)^2 } \nonumber\\
=& 15W_1 + 45W_2 + 15 W_3. \label{eq:bound-var-6-2} 
\end{align}
Combining~\eqref{eq:bound-var-6} and~\eqref{eq:bound-var-6-2}, we get
\begin{equation}\label{eq:bound-var-6-3}
\EXPS{\Delta_k^6(\vecw)} \le 243(225W_1 + 675W_2 + 225W_3 + 6040(w_k - w_k^*)^6 + 225\sigma^6).
\end{equation}
Combining~\eqref{eq:var3-lb} and~\eqref{eq:bound-var-6-3}, we get
$$
\gamma(\Delta_k(\vecw)) \le \sqrt{\frac{\EXPS{\Delta_k^6(\vecw)}}{\var(\Delta_k(\vecw))^3}} \le \sqrt{\frac{243(225W_1 + 675W_2 + 225W_3 + 6040(w_k - w_k^*)^6 + 225\sigma^6)}{W_1 + 3W_2 + W_3 + 8(w_k - w_k^*)^6 + \sigma^6}} \le 429.
$$
\end{proof}

\subsection{Proof of Proposition~\ref{thm:eg-sub-exp}}\label{prf:eg-sub-exp}
We use the same notation as in Appendix~\ref{prf:eg-var-skew-rad}. We have
\begin{align*}
\partial_kf(\vecw;\vecz) - F(\vecw) & = \Delta_k(\vecw) = \sum_{\substack{1\le i \le d \\ i\neq k}} x_kx_i(w_i - w_i^*) + (x_k^2-1)(w_k - w_k^*) - \xi x_k \\
&= x_k ( - \xi + \sum_{\substack{1\le i \le d \\ i\neq k}} x_i (w_i - w_i^*) ) := x_k\Delta'_k(\vecw)
\end{align*}
Since $\Delta'_k(\vecw)$ has symmetric distribution and $x_k$ is uniformly distributed in $\{-1, 1\}$, we know that the distributions of $\Delta_k(\vecw) $ and $\Delta'_k(\vecw)$. We then prove a stronger result on $\Delta'_k(\vecw)$. We first recall the definition of $v$-sub-Gaussian random variables. A random variable $X$ with mean $\mu = \EXPS{X}$ is $v$-sub-Gaussian if for all $\lambda\in\R$, $\EXPS{e^{\lambda(X-\mu)}} \le e^{v^2\lambda^2/2}$. We can see that $v$-sub-Gaussian random variables are also $v$-sub-exponential. One can also check that $x_i$'s are i.i.d. $1$-sub-Gaussian random variables, and then $\Delta'_k(\vecw)$ is $v$-sub-exponential with
$$
v = \big( \sigma^2 + \sum_{\substack{1\le i \le d \\ i\neq k}}(w_i - w_i^*)^2 \big)^{1/2} \le \sqrt{\sigma^2 + \twonms{\vecw - \vecw^*}^2}.
$$

\section{Proof of Theorem~\ref{thm:main-gd-sc}}\label{prf:main-gd-sc}
The proof of Theorem~\ref{thm:main-gd-sc} consists of two parts: 1) the analysis of coordinate-wise median estimator of the population gradients, and 2) the convergence analysis of the robustified gradient descent algorithm.

Recall that at iteration $t$, the master machine sends $\vecw^t$ to all the worker machines. For any normal worker machine, say machine $i\in[m]\setminus\B$, the gradient of the local empirical loss function $\vecg^i(\vecw^t) = \gradF_i(\vecw^t)$ is computed and returned to the center machine, while the Byzantine machines, say machine $i\in\B$, the returned message $\vecg^i(\vecw^t)$ can be arbitrary or even adversarial. The master machine then compute the coordinate-wise median, i.e.,
$$
\vecg(\vecw^t) = \med\{\vecg^i(\vecw^t) : i\in[m]\}.
$$
The following theorem provides a uniform bound on the distance between $\vecg(\vecw^t)$ and $\gradF(\vecw^t)$.
\begin{theorem}\label{thm:uniform-bound}
Define
\begin{equation}\label{eq:def-grad-each-machine}
\vecg^i(\vecw) = \begin{cases}
\gradF_i(\vecw) & i \in[m]\setminus\B, \\
* & i\in\B.
\end{cases}
\end{equation}
and the coordinate-wise median of $\vecg^i(\vecw)$:
\begin{equation}\label{eq:def-med-grad}
\vecg(\vecw) = \med\{\vecg^i(\vecw):i\in[m]\}.
\end{equation}
Suppose that Assumptions~\ref{asm:smoothness},~\ref{asm:bounded-variance}, and~\ref{asm:bounded-skewness} hold, and inequality~\eqref{eq:alpha-m-n-condition} is satisfied with some $\epsilon>0$.
Then, we have with probability at least $1-\frac{4d}{(1+nm\widehat{L}D)^d}$,
\begin{equation}\label{eq:median-bound-thm}
\twonms{\vecg(\vecw) - \gradF(\vecw)} \le 2\sqrt{2} \frac{1}{nm} + \sqrt{2}\frac{C_\epsilon}{\sqrt{n}}V\left(\alpha + \sqrt{\frac{d\log(1+nm\widehat{L}D)}{m(1-\alpha)}} + 0.4748\frac{S}{\sqrt{n}} \right),
\end{equation}
for all $\vecw\in\W$, where $C_\epsilon$ is defined as in~\eqref{eq:def-c-epsilon}.
\end{theorem}
\begin{proof}
See Appendix~\ref{prf:uniform-bound}.
\end{proof}
Then, we proceed to analyze the convergence of the robust distributed gradient descent algorithm. We condition on the event that the bound in~\eqref{eq:median-bound-thm} is satisfied for all $\vecw\in\W$. Then, in the $t$-th iteration, we define
$$
\widehat{\vecw}^{t+1} = \vecw^t - \eta \vecg(\vecw^t).
$$
Thus, we have $\vecw^{t+1} = \Pi_\W(\widehat{\vecw}^{t+1})$. By the property of Euclidean projection, we know that 
$$
\twonms{\vecw^{t+1} - \vecw^*} \le \twonms{\widehat{\vecw}^{t+1} - \vecw^*}.
$$
We further have
\begin{equation}\label{eq:converge-1}
\begin{aligned}
\twonms{\vecw^{t+1} - \vecw^*} & \le \twonms{\vecw^t - \eta \vecg(\vecw^t) - \vecw^*} \\
& \le \twonms{\vecw^t - \eta\gradF(\vecw^t) - \vecw^*} + \eta\twonms{\vecg(\vecw^t) -\gradF(\vecw^t)}.
\end{aligned}
\end{equation}
Meanwhile, we have
\begin{equation}\label{eq:iter-cvx-1}
\twonms{\vecw^t - \eta\gradF(\vecw^t) - \vecw^*}^2 = \twonms{\vecw^t - \vecw^*}^2 - 2\eta\innerps{\vecw^t - \vecw^*}{\gradF(\vecw^t)} + \eta^2\twonms{\gradF(\vecw^t)}^2.
\end{equation}
Since $F(\vecw)$ is $\lambda_F$-strongly convex, by the co-coercivity of strongly convex functions (see Lemma 3.11 in~\cite{bubeck2015convex} for more details), we obtain
$$
\innerps{\vecw^t - \vecw^*}{\gradF(\vecw^t)} \ge \frac{L_F\lambda_F}{L_F + \lambda_F}\twonms{\vecw^t - \vecw^*}^2 + \frac{1}{L_F + \lambda_F}\twonms{\gradF(\vecw^t)}^2.
$$
Let $\eta = \frac{1}{L_F}$. Then we get
\begin{align*}
\twonms{\vecw^t - \eta\gradF(\vecw^t) - \vecw^*}^2 & \le (1- \frac{2\lambda_F}{L_F + \lambda_F}) \twonms{\vecw^t - \vecw^*}^2 - \frac{2}{L_F(L_F+\lambda_F)}\twonms{\gradF(\vecw^t)}^2 + \frac{1}{L_F^2}\twonms{\gradF(\vecw^t)}^2 \\
& \le (1- \frac{2\lambda_F}{L_F + \lambda_F}) \twonms{\vecw^t - \vecw^*}^2,
\end{align*}
where in the second inequality we use the fact that $\lambda_F \le L_F$. Using the fact $\sqrt{1-x} \le 1-\frac{x}{2}$, we get
\begin{equation}\label{eq:iter-linear-decay}
\twonms{\vecw^t - \eta\gradF(\vecw^t) - \vecw^*} \le (1- \frac{\lambda_F}{L_F + \lambda_F})\twonms{\vecw^t - \vecw^*}.
\end{equation}
Combining~\eqref{eq:converge-1} and~\eqref{eq:iter-linear-decay}, we get
\begin{equation}\label{eq:each-iter-sc}
\twonms{\vecw^{t+1} - \vecw^*} \le (1-\frac{\lambda_F}{L_F + \lambda_F})\twonms{\vecw^t - \vecw^*} + \frac{1}{L_F}\Delta,
\end{equation}
where
$$
\Delta =  2\sqrt{2} \frac{1}{nm} + \sqrt{2}\frac{C_\epsilon}{\sqrt{n}}V (\alpha + \sqrt{\frac{d\log(1+nm\widehat{L}D)}{m(1-\alpha)}} + 0.4748\frac{S}{\sqrt{n}} ).
$$
Then we can complete the proof by iterating~\eqref{eq:each-iter-sc}.

\subsection{Proof of Theorem~\ref{thm:uniform-bound}}\label{prf:uniform-bound}
The proof of Theorem~\ref{thm:uniform-bound} relies on careful analysis of the median of means estimator in the presence of adversarial data and a covering net argument.

We first consider a general problem of robust estimation of a one dimensional random variable. Suppose that there are $m$ worker machines, and $q$ of them are Byzantine machines, which store $n$ adversarial data (recall that $\alpha := q/m$). Each of the other $m(1-\alpha)$ normal worker machines stores $n$ i.i.d. samples of some one dimensional random variable $x\sim \D$. Denote the $j$-th sample in the $i$-th worker machine by $x^{i,j}$. Let $\mu := \EXPS{x}$, $\sigma^2 := \var(x)$, and $\gamma(x)$ be the absolute skewness of $x$. In addition, define $\bar{x}^i$ as the average of samples in the $i$-th machine, i.e., $\bar{x}^i = \frac{1}{n} \sum_{j=1}^n x^{i,j}$. For any $z\in\R$, define $\widetilde{p}(z) := \frac{1}{m(1-\alpha)}\sum_{i\in[m]\setminus\B} \indi(\bar{x}^i\le z)$ as the empirical distribution function of the sample averages on the \emph{normal} worker machines. We have the following result on $\widetilde{p}(z)$.

\begin{lemma}\label{lem:quantile-one-d}
Suppose that for a fixed $t>0$, we have
\begin{equation}\label{eq:alpha-condition}
\alpha + \sqrt{\frac{t}{m(1-\alpha)}} + 0.4748\frac{\gamma(x)}{\sqrt{n}} \le \frac{1}{2}-\epsilon,
\end{equation}
for some $\epsilon>0$. Then, with probability at least $1-4e^{-2t}$, we have
\begin{equation}\label{eq:tildep1}
\widetilde{p}\left( \mu + C_\epsilon\frac{\sigma}{\sqrt{n}}(\alpha + \sqrt{\frac{t}{m(1-\alpha)}} + 0.4748\frac{\gamma(x)}{\sqrt{n}} ) \right) \ge \frac{1}{2} + \alpha,
\end{equation}
and
\begin{equation}\label{eq:tildep2}
\widetilde{p}\left( \mu - C_\epsilon\frac{\sigma}{\sqrt{n}}(\alpha + \sqrt{\frac{t}{m(1-\alpha)}} + 0.4748\frac{\gamma(x)}{\sqrt{n}} ) \right) \le \frac{1}{2} - \alpha,
\end{equation}
where $C_\epsilon$ is defined as in~\eqref{eq:def-c-epsilon}.
\end{lemma}
\begin{proof}
See Appendix~\ref{prf:quantile-one-d}.
\end{proof}
We further define the distribution function of all the $m$ machines as $\widehat{p}(z) := \frac{1}{m}\sum_{i\in[m]} \indi(\bar{x}^i\le z) $. We have the following direct corollary on $\widehat{p}(z)$ and the median of means estimator $\med\{\bar{x}^i : i\in[m]\}$.
\begin{cor}\label{cor:median-one-d}
Suppose that condition~\eqref{eq:alpha-condition} is satisfied. Then, with probability at least $1-4e^{-2t}$, we have,
\begin{equation}\label{eq:hatp1}
\widehat{p}\left( \mu + C_\epsilon\frac{\sigma}{\sqrt{n}}(\alpha + \sqrt{\frac{t}{m(1-\alpha)}} + 0.4748\frac{\gamma(x)}{\sqrt{n}} ) \right) \ge \frac{1}{2},
\end{equation}
and
\begin{equation}\label{eq:hatp2}
\widehat{p}\left( \mu - C_\epsilon\frac{\sigma}{\sqrt{n}}(\alpha + \sqrt{\frac{t}{m(1-\alpha)}} + 0.4748\frac{\gamma(x)}{\sqrt{n}} ) \right) \le \frac{1}{2}.
\end{equation}
Thus, we have with probability at least $1-4e^{-2t}$,
\begin{equation}\label{eq:median-1-d}
| \med\{\bar{x}^i : i\in[m]\} -\mu | \le C_\epsilon\frac{\sigma}{\sqrt{n}}(\alpha + \sqrt{\frac{t}{m(1-\alpha)}} + 0.4748\frac{\gamma(x)}{\sqrt{n}} ).
\end{equation}
\end{cor}
\begin{proof}
One can easily check that for any $z\in\R$, we have $| \widehat{p}(z) - \widetilde{p}(z) | \le \alpha$, which yields the results~\eqref{eq:hatp1} and~\eqref{eq:hatp2}. The result~\eqref{eq:median-1-d} can be derived using the fact that $\widehat{p}(\med\{\bar{x}^i : i\in[m]\}) = 1/2$.
\end{proof}
Lemma~\ref{lem:quantile-one-d} and Corollary~\ref{cor:median-one-d} can be translated to the estimators of the gradients. Define $\vecg^i(\vecw)$ and $\vecg(\vecw)$ as in~\eqref{eq:def-grad-each-machine} and~\eqref{eq:def-med-grad}, and let $g^i_k(\vecw)$ and $g_k(\vecw)$ be the $k$-th coordinate of $\vecg^i(\vecw)$ and $\vecg(\vecw)$, respectively. In addition, for any $\vecw\in\W$, $k\in[d]$, and $z\in\R$, we define the empirical distribution function of the $k$-th coordinate of the gradients on the normal machines:
\begin{equation}\label{eq:def-tildep-w-k}
\widetilde{p}(z; \vecw, k) = \frac{1}{m(1-\alpha)} \sum_{i\in[m]\setminus\B} \indi(g^i_k(\vecw)\le z),
\end{equation}
and on all the $m$ machines
\begin{equation}\label{eq:def-tildep-w-k}
\widehat{p}(z; \vecw, k) = \frac{1}{m} \sum_{i=1}^m \indi(g^i_k(\vecw)\le z).
\end{equation}
We use the symbol $\partial_k$ to denote the partial derivative of any function with respect to its $k$-th argument. We also use the simplified notation $\sigma_k^2(\vecw):=\var(\partial_kf(\vecw;\vecz))$, and $\gamma_k(\vecw):=\gamma(\partial_kf(\vecw;\vecz))$. Then, according to Lemma~\ref{lem:quantile-one-d}, when~\eqref{eq:alpha-condition} is satisfied, for any fixed $\vecw\in\W$ and $k\in[d]$, we have with probability at least $1-4e^{-2t}$,
\begin{equation}\label{eq:tildep-w-k1}
\widetilde{p}\left( \partial_kF(\vecw) + C_\epsilon\frac{\sigma_k(\vecw)}{\sqrt{n}}(\alpha + \sqrt{\frac{t}{m(1-\alpha)}} + 0.4748\frac{\gamma_k(\vecw)}{\sqrt{n}} ) ; \vecw, k \right) \ge \frac{1}{2} + \alpha,
\end{equation}
and
\begin{equation}\label{eq:tildep-w-k2}
\widetilde{p}\left( \partial_kF(\vecw) - C_\epsilon\frac{\sigma_k(\vecw)}{\sqrt{n}}(\alpha + \sqrt{\frac{t}{m(1-\alpha)}} + 0.4748\frac{\gamma_k(\vecw)}{\sqrt{n}} ) ; \vecw, k \right) \le \frac{1}{2} - \alpha.
\end{equation}
Further, according to Corollary~\ref{cor:median-one-d}, we know that with probability $1-4e^{-2t}$,
\begin{equation}\label{eq:bound-med}
| g_k(\vecw) - \partial_kF(\vecw) | \le C_\epsilon\frac{\sigma_k(\vecw)}{\sqrt{n}}(\alpha + \sqrt{\frac{t}{m(1-\alpha)}} + 0.4748\frac{\gamma_k(\vecw)}{\sqrt{n}} ).
\end{equation}
Here, the inequality~\eqref{eq:bound-med} gives a bound on the accuracy of the median of means estimator for the gradient at any fixed $\vecw$ and any coordinate $k\in[d]$. To extend this result to all $\vecw\in\W$ and all the $d$ coordinates, we need to use union bound and a covering net argument.

Let $\W_\delta = \{\vecw^1,\vecw^2,\ldots,\vecw^{N_\delta}\}$ be a finite subset of $\W$ such that for any $\vecw\in\W$, there exists $\vecw^\ell \in \W_\delta$ such that $\twonms{\vecw^\ell - \vecw} \le \delta$. According to the standard covering net results~\citep{vershynin2010introduction}, we know that $N_\delta \le (1+\frac{D}{\delta})^d$. By union bound, we know that with probability at least $1-4dN_\delta e^{-2t}$, the bounds in~\eqref{eq:tildep-w-k1} and~\eqref{eq:tildep-w-k2} hold for all $\vecw = \vecw^\ell \in\W_\delta$, and $k\in[d]$. By gathering all the $k$ coordinates and using Assumption~\ref{asm:bounded-skewness}, we know that this implies for all $\vecw^\ell\in\W_\delta$,
\begin{equation}\label{eq:median-acc-in-cover}
\twonms{ \vecg( \vecw^\ell ) - \gradF( \vecw^\ell ) } \le \frac{C_\epsilon}{\sqrt{n}} V \left( \alpha + \sqrt{\frac{t}{m(1-\alpha)}} + 0.4748\frac{S}{\sqrt{n}} \right).
\end{equation}

Then, consider an arbitrary $\vecw\in\W$. Suppose that $\twonms{\vecw^\ell - \vecw} \le \delta$. Since by Assumption~\ref{asm:smoothness}, we assume that for each $k\in[d]$, the partial derivative $\partial_k f(\vecw; \vecz)$ is $L_k$-Lipschitz for all $\vecz$, we know that for every normal machine $i\in[m]\setminus\B$, 
$$
\abs{g^i_k(\vecw) - g^i_k(\vecw^\ell)} \le L_k\delta.
$$
Then, according to the definition of $\widetilde{p}(z;\vecw,k)$ in~\eqref{eq:def-tildep-w-k}, we know that for any $z\in\R$, $\widetilde{p}(z+L_k\delta;\vecw,k) \ge \widetilde{p}(z;\vecw^\ell,k)$ and $\widetilde{p}(z-L_k\delta;\vecw,k)\le \widetilde{p}(z;\vecw^\ell,k)$. Then, the bounds in~\eqref{eq:tildep-w-k1} and~\eqref{eq:tildep-w-k2} yield
\begin{equation}\label{eq:tildep-anyw-k-1}
\widetilde{p}\left( \partial_kF(\vecw^\ell) + L_k\delta + C_\epsilon\frac{\sigma_k(\vecw^\ell)}{\sqrt{n}}(\alpha + \sqrt{\frac{t}{m(1-\alpha)}} + 0.4748\frac{\gamma_k(\vecw^\ell)}{\sqrt{n}} ) ; \vecw, k \right) \ge \frac{1}{2} + \alpha,
\end{equation}
and
\begin{equation}\label{eq:tildep-anyw-k-k2}
\widetilde{p}\left( \partial_kF(\vecw^\ell) - L_k\delta - C_\epsilon\frac{\sigma_k(\vecw^\ell)}{\sqrt{n}}(\alpha + \sqrt{\frac{t}{m(1-\alpha)}} + 0.4748\frac{\gamma_k(\vecw^\ell)}{\sqrt{n}} ) ; \vecw, k \right) \le \frac{1}{2} - \alpha.
\end{equation}
Using the fact that $\abss{\partial_kF(\vecw^\ell) - \partial_kF(\vecw)} \le L_k\delta$, and Corollary~\ref{cor:median-one-d}, we have
$$
| g_k(\vecw) - \partial_kF(\vecw) | \le 2L_k\delta + C_\epsilon\frac{\sigma_k(\vecw^\ell)}{\sqrt{n}}(\alpha + \sqrt{\frac{t}{m(1-\alpha)}} + 0.4748\frac{\gamma_k(\vecw^\ell)}{\sqrt{n}} ).
$$
Again, by gathering all the $k$ coordinates we get
$$
\twonms{\vecg(\vecw) - \gradF(\vecw)}^2 \le 8\delta^2\sum_{k=1}^d L_k^2 + 2\frac{C_\epsilon^2}{n}\sum_{k=1}^d \sigma_k^2(\vecw^\ell)(\alpha + \sqrt{\frac{t}{m(1-\alpha)}} + 0.4748\frac{\gamma_k(\vecw^\ell)}{\sqrt{n}} )^2,
$$
where we use the fact that $(a+b)^2 \le 2(a^2+b^2)$. Then, by Assumption~\ref{asm:bounded-variance} and~\ref{asm:bounded-skewness}, we further obtain
\begin{equation}\label{eq:median-acc-all-w}
\twonms{\vecg(\vecw) - \gradF(\vecw)} \le 2\sqrt{2}\delta \widehat{L} + \sqrt{2}\frac{C_\epsilon}{\sqrt{n}}V\left(\alpha + \sqrt{\frac{t}{m(1-\alpha)}} + 0.4748\frac{S}{\sqrt{n}} \right),
\end{equation}
where we use the fact that $\sqrt{a+b} \le \sqrt{a} + \sqrt{b}$. Combining~\eqref{eq:median-acc-in-cover} and~\eqref{eq:median-acc-all-w}, we conclude that for any $\delta>0$, with probability at least $1-4dN_\delta e^{-2t}$,~\eqref{eq:median-acc-all-w} holds for all $\vecw\in\W$. We simply choose $\delta = \frac{1}{nm\widehat{L}}$, and $t=d\log(1+nm\widehat{L}D)$. Then, we know that with probability at least $1-\frac{4d}{(1+nm\widehat{L}D)^d}$, we have
$$
\twonms{\vecg(\vecw) - \gradF(\vecw)} \le 2\sqrt{2} \frac{1}{nm} + \sqrt{2}\frac{C_\epsilon}{\sqrt{n}}V\left(\alpha + \sqrt{\frac{d\log(1+nm\widehat{L}D)}{m(1-\alpha)}} + 0.4748\frac{S}{\sqrt{n}} \right)
$$
for all $\vecw\in\W$.

\subsection{Proof of Lemma~\ref{lem:quantile-one-d}}\label{prf:quantile-one-d}

We recall the Berry-Esseen Theorem~\citep{berry1941accuracy,esseen1942liapounoff,shevtsova2014absolute} and the bounded difference inequality, which are useful in this proof.
\begin{theorem}[Berry-Esseen Theorem]\label{thm:berry-Esseen}
Assume that $Y_1,\ldots,Y_n$ are i.i.d. copies of a random variable $Y$ with mean $\mu$, variance $\sigma^2$, and such that $\EXPS{|Y-\mu|^3} < \infty$. Then,
$$
\sup_{s\in\R} \abs{ \prob{ \sqrt{n}\frac{\bar{Y}-\mu}{\sigma} \le s } - \Phi(s) } \le 0.4748 \frac{\EXPS{|Y-\mu|^3}}{\sigma^3\sqrt{n}},
$$
where $\bar{Y} = \frac{1}{n}\sum_{i=1}^n Y_i$ and $\Phi(s)$ is the cumulative distribution function of the standard normal random variable.
\end{theorem}
\begin{theorem}[Bounded Difference Inequality]\label{thm:bounded-difference}
Let $X_1,\ldots,X_n$ be i.i.d. random variables, and assume that $Z = g(X_1,\ldots, X_n)$, where $g$ satisfies that for all $j\in[n]$ and
all $x_1, x_2, \ldots, x_j, x'_j, \ldots, x_n$,
$$
|g(x_1,\ldots,x_{j-1}, x_j, x_{j+1}, \ldots, x_n) - g(x_1,\ldots,x_{j-1}, x'_j, x_{j+1}, \ldots, x_n)| \le c_j.
$$
Then for any $t\ge 0$,
$$
\prob{Z-\EXPS{Z} \ge t} \le \exp\left( -\frac{2t^2}{\sum_{j=1}^n c_j^2} \right)
$$
and
$$
\prob{Z-\EXPS{Z} \le -t} \le \exp\left( -\frac{2t^2}{\sum_{j=1}^n c_j^2} \right).
$$
\end{theorem}
Let $\sigma_n := \frac{\sigma}{\sqrt{n}}$ and $c_n := 0.4748 \frac{\EXPS{|x-\mu|^3}}{\sigma^3\sqrt{n}} = 0.4748\frac{\gamma(x)}{\sqrt{n}}$. Define $W_i:=\frac{\bar{x}^{i} - \mu}{\sigma_n}$ for all $i\in[m]$, and $\Phi_n(\cdot)$ be the distribution function of $W_i$ for any $i \in [m]\setminus\B$. We also define the empirical distribution function of $\{W_i:i\in[m]\setminus\B\}$ as $\widetilde{\Phi}_n(\cdot)$, i.e., $\widetilde{\Phi}_n(z) = \frac{1}{m(1-\alpha)}\sum_{i\in[m]\setminus\B}\indi(W_i\le z)$. Thus, we have
\begin{equation}\label{eq:phi-and-p}
\widetilde{\Phi}_n(z) = \widetilde{p}(\sigma_n z+\mu).
\end{equation}
We then focus on $\widetilde{\Phi}_n(z)$. We know that for any $z\in\R$, $\EXPS{\widetilde{\Phi}_n(z)} = \Phi_n(z)$. Then, since the bounded difference inequality is satisfied with $c_j = \frac{1}{m(1-\alpha)}$, we have for any $t>0$,
\begin{equation}\label{eq:bound-empirical}
\abs{\widetilde{\Phi}_n(z) - \Phi_n(z)} \le \sqrt{\frac{t}{m(1-\alpha)}},
\end{equation} 
on the draw of $W_i$, $i\in[m]\setminus\B$ with probability at least $1-2e^{-2t}$.
Let $z_1\ge z_2$ be such that $\Phi_n(z_1) \ge \frac{1}{2} + \alpha + \sqrt{\frac{t}{m(1-\alpha)}}$, and $\Phi_n(z_2) \le \frac{1}{2} - \alpha - \sqrt{\frac{t}{m(1-\alpha)}}$. Then, by union bound, we know that with probability at least $1-4e^{-2t}$, $\widetilde{\Phi}_n(z_1)\ge 1/2 + \alpha$ and $\widetilde{\Phi}_n(z_2)\le 1/2-\alpha$. The next step is to choose $z_1$ and $z_2$. According to Theorem~\ref{thm:berry-Esseen}, we know that
$$
\Phi_n(z_1) \ge \Phi(z_1) - c_n,
$$
and thus, it suffices to find $z_1$ such that
$$
\Phi(z_1) = \frac{1}{2} + \alpha + \sqrt{\frac{t}{m(1-\alpha)}} + c_n.
$$
By mean value theorem, we know that there exists $\xi\in[0, z_1]$ such that
$$
\alpha + \sqrt{\frac{t}{m(1-\alpha)}} + c_n = z_1\Phi'(\xi) = \frac{z_1}{\sqrt{2\pi}}e^{-\frac{\xi^2}{2}} \ge \frac{z_1}{\sqrt{2\pi}}e^{-\frac{z_1^2}{2}}
$$
Suppose that for some fix constant $\epsilon\in(0,1/2)$, we have
$$
\alpha + \sqrt{\frac{t}{m(1-\alpha)}} + c_n \le \frac{1}{2}-\epsilon.
$$
Then, we know that $z_1\le \Phi^{-1}(1-\epsilon)$, and thus we have
$$
\alpha + \sqrt{\frac{t}{m(1-\alpha)}} + c_n \ge \frac{z_1}{\sqrt{2\pi}}\exp(-\frac{1}{2}(\Phi^{-1}(1-\epsilon))^2),
$$
which yields
$$
z_1 \le \sqrt{2\pi}\exp(\frac{1}{2}(\Phi^{-1}(1-\epsilon))^2) \left( \alpha + \sqrt{\frac{t}{m(1-\alpha)}} + c_n \right).
$$
Similarly
$$
z_2 \ge -\sqrt{2\pi}\exp(\frac{1}{2}(\Phi^{-1}(1-\epsilon))^2) \left( \alpha + \sqrt{\frac{t}{m(1-\alpha)}} + c_n \right).
$$
For simplicity, let $C_\epsilon := \sqrt{2\pi}\exp(\frac{1}{2}(\Phi^{-1}(1-\epsilon))^2)$. We conclude that with probability $1-4e^{-2t}$, we have
$$
\widetilde{p}( \mu + C_\epsilon\sigma_n(\alpha + \sqrt{\frac{t}{m(1-\alpha)}} + c_n) ) \ge \frac{1}{2} + \alpha,
$$
and 
$$
\widetilde{p}( \mu - C_\epsilon\sigma_n(\alpha + \sqrt{\frac{t}{m(1-\alpha)}} + c_n) ) \le \frac{1}{2} - \alpha.
$$

\section{Proof of Theorem~\ref{thm:main-gd-cvx}}\label{prf:main-gd-cvx}
Since Theorem~\ref{thm:uniform-bound} holds without assuming the convexity of $F(\vecw)$, when $F(\vecw)$ is non-strongly convex, the event that~\eqref{eq:median-bound-thm} holds for all $\vecw\in\W$ still happens with probability at least $1-\frac{4d}{(1+nm\widehat{L}D)^d}$. We condition on this event. We first show that when Assumption~\ref{asm:size-of-space} is satisfied and we choose $\eta = \frac{1}{L_F}$, the iterates $\vecw^t$ stays in $\W$ without using projection. Namely, define
$$ 
\vecw^{t+1} = \vecw^t - \eta \vecg(\vecw^t),
$$
for $T=0,1,\ldots, T-1$, then $\vecw^{t}\in\W$ for all $t=0,1,\ldots, T$.
To see this, we have
$$
\twonms{\vecw^{t+1} - \vecw^*} \le \twonms{\vecw^t - \eta \gradF(\vecw^t) -\vecw^*} + \eta\twonms{\vecg(\vecw^t) - \gradF(\vecw^t)},
$$
and
\begin{align*}
\twonms{\vecw^t - \eta \gradF(\vecw^t) -\vecw^*}^2 &= \twonms{\vecw^t - \vecw^*}^2 - 2\eta\innerps{\gradF(\vecw^t)}{\vecw^t - \vecw^*} + \eta^2\twonms{\gradF(\vecw^t)}^2 \\
&\le \twonms{\vecw^t - \vecw^*}^2 -2\eta \frac{1}{L_F}\twonms{\gradF(\vecw^t)}^2 + \eta^2\twonms{\gradF(\vecw^t)}^2 \\
&= \twonms{\vecw^t - \vecw^*}^2 - \frac{1}{L_F^2}\twonms{\gradF(\vecw^t)}^2 \\
&\le \twonms{\vecw^t - \vecw^*}^2
\end{align*}
where the inequality is due to the co-coercivity of convex functions. Thus, we get
$$
\twonms{\vecw^{t+1} - \vecw^*} \le \twonms{\vecw^{t} - \vecw^*}  + \frac{\Delta}{L_F},
$$
and since $T = \frac{L_FD_0}{\Delta}$, according to Assumption~\ref{asm:size-of-space} we know that $\vecw^t\in\W$ for all $t=0,1,\ldots, T$. Then, we proceed to study the algorithm without projection. Here, we define $D_t:=\twonms{\vecw^0 - \vecw^*} + \frac{t \Delta}{L_F}$ for $t=0,1,\ldots, T$. 

Using the smoothness of $F(\vecw)$, we have
\begin{align*}
F(\vecw^{t+1}) & \le F(\vecw^t) + \innerps{\gradF(\vecw^t)}{\vecw^{t+1} - \vecw^t} + \frac{L_F}{2}\twonms{\vecw^{t+1} - \vecw^t}^2 \\
&= F(\vecw^t) + \eta \innerps{\gradF(\vecw^t)}{ - \vecg(\vecw^t) + \gradF(\vecw^t) - \gradF(\vecw^t)} + \eta^2 \frac{L_F}{2} \twonms{ \vecg(\vecw^t) - \gradF(\vecw^t) + \gradF(\vecw^t)}^2.
\end{align*}
Since $\eta = \frac{1}{L_F}$ and $\twonms{ \vecg(\vecw^t) - \gradF(\vecw^t)} \le \Delta$, by simple algebra, we obtain
\begin{equation}\label{eq:smoothness-noise}
F(\vecw^{t+1}) \le F(\vecw^t) - \frac{1}{2L_F}\twonms{\gradF(\vecw^t)}^2 + \frac{1}{2L_F}\Delta^2.
\end{equation}
We now prove the following lemma.
\begin{lemma}\label{lem:cvx-first-part}
Condition on the event that~\eqref{eq:median-bound-thm} holds for all $\vecw\in\W$. When $F(\vecw)$ is convex, by running $T =  \frac{L_FD_0}{\Delta}$ parallel iterations, there exists $t\in\{0,1,2,\ldots, T\}$ such that
$$
F(\vecw^t) - F(\vecw^*) \le 16D_0\Delta.
$$
\end{lemma}
\begin{proof}
We first notice that since $T = \frac{L_FD_0}{\Delta}$, we have $D_t \le 2D_0$ for all $t=0,1,\ldots, T$. According to the first order optimality of convex functions, for any $\vecw$,
$$
F(\vecw) - F(\vecw^*) \le \innerps{\gradF(\vecw)}{\vecw - \vecw^*} \le \twonms{\gradF(\vecw)}\twonms{\vecw - \vecw^*},
$$
and thus
\begin{equation}\label{eq:lb-norm-grad}
\twonms{\gradF(\vecw)} \ge \frac{F(\vecw) - F(\vecw^*)}{\twonms{\vecw - \vecw^*}}.
\end{equation}
Suppose that there exists $t\in\{0,1,\ldots, T-1\}$ such that $\twonms{\gradF(\vecw^t)} < \sqrt{2}\Delta$. Then we have 
$$
F(\vecw^t) - F(\vecw^*) \le \twonms{\gradF(\vecw^t)} \twonms{\vecw^t - \vecw^*} \le 2\sqrt{2}D_0\Delta.
$$
Otherwise, for all $t\in\{0,1,\ldots, T-1\}$, $\twonms{\gradF(\vecw^t)} \ge \sqrt{2}\Delta$. Then, according to~\eqref{eq:smoothness-noise} and~\eqref{eq:lb-norm-grad}, we have for all $t < T$,
\begin{align*}
F(\vecw^{t+1}) - F(\vecw^*) &\le F(\vecw^t) -F(\vecw^*)- \frac{1}{4L_F}\twonms{\gradF(\vecw^t)}^2 \\
&\le F(\vecw^t) -F(\vecw^*) - \frac{1}{4L_FD_t^2}(F(\vecw^t) -F(\vecw^*))^2.
\end{align*}
Multiplying both sides by $[(F(\vecw^{t+1}) - F(\vecw^*)) (F(\vecw^t) -F(\vecw^*)]^{-1}$ and rearranging the terms, we obtain
$$
\frac{1}{F(\vecw^{t+1}) - F(\vecw^*)} \ge \frac{1}{F(\vecw^t) -F(\vecw^*)} + \frac{1}{4L_FD_t^2}\frac{F(\vecw^t) -F(\vecw^*)}{F(\vecw^{t+1}) -F(\vecw^*)} \ge \frac{1}{F(\vecw^t) -F(\vecw^*)} + \frac{1}{16L_FD_0^2},
$$
which implies
$$
\frac{1}{F(\vecw^{T}) - F(\vecw^*)} \ge \frac{1}{F(\vecw^{0}) - F(\vecw^*)} + \frac{T}{16L_FD_0^2} \ge \frac{T}{16L_FD_0^2}.
$$
Then, we obtain $F(\vecw^{T}) - F(\vecw^*) \le 16 D_0\Delta$ using the fact that $T = \frac{L_FD_0}{\Delta}$.
\end{proof}
Next, we show that $F(\vecw^T) - F(\vecw^*) \le  16 D_0\Delta + \frac{1}{2L_F}\Delta^2$. More specifically, let $t=t_0$ be the first time that $F(\vecw^{t}) - F(\vecw^*) \le 16 D_0\Delta$, and we show that for any $t > t_0$, $F(\vecw^{t}) - F(\vecw^*) \le 16 D_0\Delta + \frac{1}{2L_F}\Delta^2$. If this statement is not true, then we let $t_1 > t_0$ be the first time that $F(\vecw^{t}) - F(\vecw^*) > 16 D_0\Delta + \frac{1}{2L_F}\Delta^2$. Then there must be $F(\vecw^{t_1-1}) < F(\vecw^{t_1})$. According to~\eqref{eq:smoothness-noise}, there should also be 
$$
F(\vecw^{t_1-1}) - F(\vecw^*) \ge F(\vecw^{t_1}) - F(\vecw^*) - \frac{1}{2L_F}\Delta^2 > 16D_0\Delta.
$$
Then, according to~\eqref{eq:lb-norm-grad}, we have
$$
\twonms{\gradF(\vecw^{t_1-1})} \ge \frac{F(\vecw^{t_1-1}) - F(\vecw^*)}{\twonms{\vecw^{t_1-1} -\vecw^* }} > 8\Delta.
$$
Then according to~\eqref{eq:smoothness-noise}, this implies $ F(\vecw^{t_1}) \le F(\vecw^{t_1-1}) $,
which contradicts with the fact that $F(\vecw^{t_1-1}) < F(\vecw^{t_1})$.

\section{Proof of Theorem~\ref{thm:main-gd-non-cvx}}\label{prf:main-gd-non-cvx}
Since Theorem~\ref{thm:uniform-bound} holds without assuming the convexity of $F(\vecw)$, when $F(\vecw)$ is non-convex, the event that~\eqref{eq:median-bound-thm} holds for all $\vecw\in\W$ still happens with probability at least $1-\frac{4d}{(1+nm\widehat{L}D)^d}$. We condition on this event. We first show that when Assumption~\ref{asm:size-of-space-non-cvx} is satisfied and we choose $\eta = \frac{1}{L_F}$, the iterates $\vecw^t$ stays in $\W$ without using projection. Since we have
$$
\twonms{\vecw^{t+1} - \vecw^*} \le \twonms{\vecw^t - \vecw^*} + \eta(\twonms{\gradF(\vecw^t)} + \twonms{\vecg(\vecw^t) - \gradF(\vecw^t)}) \le \twonms{\vecw^t - \vecw^*}  + \frac{1}{L_F}(M+\Delta).
$$
Then, we know that by running $T = \frac{2L_F}{\Delta^2}(F(\vecw^0) - F(\vecw^*))$ parallel iterations, using Assumption~\ref{asm:size-of-space-non-cvx}, we know that $\vecw^t \in \W$ for $t=0,1,\ldots, T$ without projection.

We proceed to study the convergence rate of the algorithm. By the smoothness of $F(\vecw)$, we know that when choosing $\eta = \frac{1}{L_F}$, the inequality~\eqref{eq:smoothness-noise} still holds. More specifically, for all $t=0,1,\ldots, T-1$,
\begin{equation}\label{eq:smoothness-noise-2}
F(\vecw^{t+1}) - F(\vecw^*) \le F(\vecw^t) - F(\vecw^*) - \frac{1}{2L_F}\twonms{\gradF(\vecw^t)}^2 + \frac{1}{2L_F}\Delta^2.
\end{equation}
Sum up~\eqref{eq:smoothness-noise-2} for $t=0,1,\ldots, T-1$. Then, we get
$$
0 \le F(\vecw^T) - F(\vecw^*) \le F(\vecw^0) - F(\vecw^*) - \frac{1}{2L_F}\sum_{t=0}^{T-1}\twonms{\gradF(\vecw^t)}^2 + \frac{T}{2L_F}\Delta^2.
$$
This implies that
$$
\min_{t=0,1,\ldots, T} \twonms{\gradF(\vecw^t)}^2  \le 2\frac{L_F}{T} (F(\vecw^0) - F(\vecw^*)) + \Delta^2,
$$
which completes the proof.

\section{Proof of Theorem~\ref{thm:main-gd-sc-trim}}\label{prf:main-gd-sc-trim}

The proof of Theorem~\ref{thm:main-gd-sc-trim} consists of two parts: 1) the analysis of coordinate-wise trimmed mean of means estimator of the population gradients, and 2) the convergence analysis of the robustified gradient descent algorithm. Since the second part is essentially the same as the proof of Theorem~\ref{thm:main-gd-sc}, we mainly focus on the first part here. 

\begin{theorem}\label{thm:uniform-bound-trim}
Define
\begin{equation}\label{eq:def-grad-each-machine-trim}
\vecg^i(\vecw) = \begin{cases}
\gradF_i(\vecw) & i \in[m]\setminus\B, \\
* & i\in\B.
\end{cases}
\end{equation}
and the coordinate-wise trimmed mean of $\vecg^i(\vecw)$:
\begin{equation}\label{eq:def-trim-grad}
\vecg(\vecw) = \trim_\beta\{\vecg^i(\vecw):i\in[m]\}.
\end{equation} 
Suppose that Assumptions~\ref{asm:smoothness} and~\ref{asm:sub-exponential} are satisfied, and that $\alpha \le \beta \le \frac{1}{2} - \epsilon$. Then, with probability at least $1-\frac{2d(m+1)}{(1+nm\widehat{L}D)^d}$,
$$
\twonms{\vecg(\vecw) - \gradF(\vecw)} \le \frac{v}{\epsilon} \left( \frac{3\sqrt{2}\beta d}{\sqrt{n}} + \frac{2d}{\sqrt{nm}}  \right)\sqrt{\log(1+nm\widehat{L}D) + \frac{1}{d}\log m} + \widetilde{\bigo}(\frac{\beta}{n} + \frac{1}{nm})
$$
for all $\vecw \in \W$.
\end{theorem}
\begin{proof}
See Appendix~\ref{prf:uniform-bound-trim}
\end{proof}
The rest of the proof is essentially the same as the proof of Theorem~\ref{thm:main-gd-sc}. In fact, we essentially analyze a gradient descent algorithm with bounded noise in the gradients. In the proof of Theorem~\ref{thm:main-gd-sc} in Appendix~\ref{prf:main-gd-sc}. The bound on the noise in the gradients is
$$
\Delta = \sqrt{2}\frac{C_\epsilon}{\sqrt{n}}V (\alpha + \sqrt{\frac{d\log(1+nm\widehat{L}D)}{m(1-\alpha)}} + 0.4748\frac{S}{\sqrt{n}}) + 2\sqrt{2} \frac{1}{nm},
$$
while here we replace $\Delta$ with $\Delta'$:
$$
\Delta' := \frac{v}{\epsilon} \left( \frac{3\sqrt{2}\beta d}{\sqrt{n}} + \frac{2d}{\sqrt{nm}}  \right)\sqrt{\log(1+nm\widehat{L}D) + \frac{1}{d}\log m} + \widetilde{\bigo}(\frac{\beta}{n} + \frac{1}{nm}),
$$
and the same analysis can still go through. Therefore, we omit the details of the analysis here. 
\begin{remark}\label{rmk:non-strongly-non-cvx}
The same arguments still go through when the population risk function $F(\vecw)$ is non-strongly convex or non-convex. One can simply replace the bound on the noise in the gradients $\Delta$ in Theorems~\ref{thm:main-gd-cvx} and~\ref{thm:main-gd-non-cvx} with $\Delta'$ here. Thus we omit the details here.
\end{remark}

\subsection{Proof of Theorem~\ref{thm:uniform-bound-trim}}\label{prf:uniform-bound-trim}
The proof of Theorem~\ref{thm:uniform-bound-trim} relies on the analysis of the trimmed mean of means estimator in the presence of adversarial data and a covering net argument. We first consider a general problem of robust estimation of a one dimensional random variable. Suppose that there are $m$ worker machines, and $q$ of them are Byzantine machines, which store $n$ adversarial data (recall that $\alpha := q/m$). Each of the other $m(1-\alpha)$ normal worker machines stores $n$ i.i.d. samples of some one dimensional random variable $x\sim \D$. Suppose that $x$ is $v$-sub-exponential and let $\mu := \EXPS{x}$. Denote the $j$-th sample in the $i$-th worker machine by $x^{i,j}$. In addition, define $\bar{x}^i$ as the average of samples in the $i$-th machine, i.e., $\bar{x}^i = \frac{1}{n} \sum_{j=1}^n x^{i,j}$. We have the following result on the trimmed mean of $\bar{x}^i$, $i\in[m]$.

\begin{lemma}\label{lem:quantile-one-d-trim}
Suppose that the one dimensional samples on all the normal machines are i.i.d. $v$-sub-exponential with mean $\mu$. Then, we have for any $t \ge 0$,
$$
\probs{\abss{\frac{1}{(1-\alpha)m}\sum_{i\in [m]\setminus\B} \bar{x}^i - \mu} \ge t} \le 2\exp\{-(1-\alpha)mn\min\{\frac{t}{2v}, \frac{t^2}{2v^2}\}\},
$$
and for any $s\ge 0$,
$$
\probs{\max_{i\in[m]\setminus\B}\{ \abss{ \bar{x}^i - \mu} \} \ge s} \le 2(1-\alpha)m\exp\{-n\min\{\frac{s}{2v}, \frac{s^2}{2v^2}\}\},
$$
and when $\beta \ge \alpha$, $\abss{\frac{1}{(1-\alpha)m}\sum_{i\in [m]\setminus\B} \bar{x}^i - \mu} \le t$, and $\max_{i\in[m]\setminus\B}\{ \abss{ \bar{x}^i - \mu} \} \le s$,
we have
$$
\abss{ \trim_\beta\{\bar{x}^i : i\in[m]\} - \mu  } \le \frac{t + 3\beta s}{1-2\beta}.
$$
\end{lemma}

\begin{proof}
See Appendix~\ref{prf:quantile-one-d-trim}.
\end{proof}

Lemma~\ref{lem:quantile-one-d-trim} can be directly applied to the $k$-th partial derivative of the loss functions.
Since we assume that for any $k\in[d]$ and $\vecw\in\W$, $\partial_k f(\vecw;\vecz)$ is $v$-sub-exponential, we have for any $t\ge0$, $s \ge 0$,
\begin{equation}\label{eq:bound-all-normal-1}
\probs{\abss{\frac{1}{(1-\alpha)m}\sum_{i\in [m]\setminus\B} g^i_k(\vecw) - \partial_k F(\vecw)} \ge t} \le 2\exp\{-(1-\alpha)mn\min\{\frac{t}{2v}, \frac{t^2}{2v^2}\}\},
\end{equation}
\begin{equation}\label{eq:bound-max-normal-1}
\probs{\max_{i\in[m]\setminus\B}\{ \abss{ g^i_k(\vecw) - \partial_k F(\vecw)} \} \ge s} \le 2(1-\alpha)m\exp\{-n\min\{\frac{s}{2v}, \frac{s^2}{2v^2}\}\},
\end{equation}
and consequently with probability at least
$$
1 - 2\exp\{-(1-\alpha)mn\min\{\frac{t}{2v}, \frac{t^2}{2v^2}\}\} - 2(1-\alpha)m\exp\{-n\min\{\frac{s}{2v}, \frac{s^2}{2v^2}\}\},
$$
we have
\begin{equation}\label{eq:bound-fix-w-k}
\abss{g_k(\vecw) - \partial_k F(\vecw)} = \abss{ \trim_\beta\{g^i_k(\vecw) : i\in[m]\} - \partial_k F(\vecw) } \le \frac{t + 3\beta s}{1-2\beta}.
\end{equation}
To extend this result to all $\vecw\in\W$ and all the $d$ coordinates, we need to use union bound and a covering net argument.
Let $\W_\delta = \{\vecw^1,\vecw^2,\ldots,\vecw^{N_\delta}\}$ be a finite subset of $\W$ such that for any $\vecw\in\W$, there exists $\vecw^\ell \in \W_\delta$ such that $\twonms{\vecw^\ell - \vecw} \le \delta$. According to the standard covering net results~\citep{vershynin2010introduction}, we know that $N_\delta \le (1+\frac{D}{\delta})^d$. By union bound, we know that with probability at least 
$$
1-2dN_\delta\exp\{-(1-\alpha)mn\min\{\frac{t}{2v}, \frac{t^2}{2v^2}\}\},
$$
the bound $\abss{\frac{1}{(1-\alpha)m}\sum_{i\in [m]\setminus\B} g^i_k(\vecw) - \partial_k F(\vecw)} \le t$ holds for all $\vecw = \vecw^\ell \in\W_\delta$, and $k\in[d]$,
and with probability at least
$$
1 - 2(1-\alpha)dmN_\delta\exp\{-n\min\{\frac{s}{2v}, \frac{s^2}{2v^2}\}\}
$$
the bound $\max_{i\in[m]\setminus\B}\{ \abss{ g^i_k(\vecw) - \partial_k F(\vecw)} \} \le s$ holds for all $\vecw = \vecw^\ell \in\W_\delta$, and $k\in[d]$. By gathering all the $k$ coordinates, we know that this implies for all $\vecw^\ell\in\W_\delta$,
\begin{equation}\label{eq:trim-acc-in-cover}
\twonms{ \vecg( \vecw^\ell ) - \gradF( \vecw^\ell ) } \le \sqrt{d}\frac{t + 3\beta s}{1-2\beta}.
\end{equation}
Then, consider an arbitrary $\vecw\in\W$. Suppose that $\twonms{\vecw^\ell - \vecw} \le \delta$. Since by Assumption~\ref{asm:smoothness}, we assume that for each $k\in[d]$, the partial derivative $\partial_k f(\vecw; \vecz)$ is $L_k$-Lipschitz for all $\vecz$, we know that for every normal machine $i\in[m]\setminus\B$, 
$$
\abs{g^i_k(\vecw) - g^i_k(\vecw^\ell)} \le L_k\delta,\quad \abs{\partial_k F(\vecw) - \partial_k F(\vecw^\ell)} \le L_k\delta.
$$
This means that if $\abss{\frac{1}{(1-\alpha)m}\sum_{i\in [m]\setminus\B} g^i_k(\vecw^\ell) - \partial_k F(\vecw^\ell)} \le t$ and $\max_{i\in[m]\setminus\B}\{ \abss{ g^i_k(\vecw^\ell) - \partial_k F(\vecw^\ell)} \} \le s$ hold for all $\vecw^\ell \in\W_\delta$, and $k\in[d]$, then
$$
\abss{\frac{1}{(1-\alpha)m}\sum_{i\in [m]\setminus\B} g^i_k(\vecw) - \partial_k F(\vecw)} \le t + 2L_k\delta,
$$
and
$$
\max_{i\in[m]\setminus\B}\{ \abss{ g^i_k(\vecw) - \partial_k F(\vecw)} \} \le s + 2L_k\delta
$$
hold for all $\vecw\in\W$. This implies that for all $\vecw\in\W$ and $k\in[d]$,
$$
\abss{g_k(\vecw) - \partial_k F(\vecw)} = \abss{ \trim_\beta\{g^i_k(\vecw) : i\in[m]\} - \partial_k F(\vecw) } \le \frac{t + 3\beta s}{1-2\beta} + \frac{2(1+3\beta)}{1-2\beta} \delta L_k,
$$
which yields
$$
\twonms{\vecg(\vecw) - \gradF(\vecw)} \le \sqrt{2d} \frac{t + 3\beta s}{1-2\beta} + \sqrt{2} \frac{2(1+3\beta)}{1-2\beta} \delta \widehat{L}.
$$
The proof is completed by choosing $\delta = \frac{1}{nm\widehat{L}}$,
\begin{align*}
t &= v \max\{\frac{8d}{nm}\log(1+nm\widehat{L}D), \sqrt{ \frac{8d}{nm}\log(1+nm\widehat{L}D) } \}, \\
s & = v \max\{ \frac{4}{n}( d\log(1+nm\widehat{L}D) +\log m ) , \sqrt{\frac{4}{n}(d\log(1+nm\widehat{L}D) + \log m) } \},
\end{align*}
and using the fact that $\beta \le \frac{1}{2} - \epsilon$.

\subsection{Proof of Lemma~\ref{lem:quantile-one-d-trim}}\label{prf:quantile-one-d-trim}
We first recall Bernstein's inequality for sub-exponential random variables.
\begin{theorem}[Bernstein's inequality]\label{thm:bernstein}
Suppose that $X_1,X_2,\ldots,X_n$ are i.i.d. $v$-sub-exponential random variables with mean $\mu$. Then for any $t\ge 0$,
$$
\probs{\abss{\frac{1}{n}\sum_{i=1}^n X_i - \mu} \ge t}\le 2\exp\{-n\min\{\frac{t}{2v}, \frac{t^2}{2v^2}\}\}.
$$
\end{theorem}
Thus, for any $t\ge 0$
\begin{equation}\label{eq:bound-all-machine}
\probs{\abss{\frac{1}{(1-\alpha)m}\sum_{i\in [m]\setminus\B} \bar{x}^i - \mu} \ge t} \le 2\exp\{-(1-\alpha)mn\min\{\frac{t}{2v}, \frac{t^2}{2v^2}\}\}.
\end{equation}
Similarly, for any $i\in[m]\setminus\B$, and any $s\ge 0$
$$
\probs{\abss{ \bar{x}^i - \mu} \ge s} \le 2\exp\{-n\min\{\frac{s}{2v}, \frac{s^2}{2v^2}\}\}.
$$
Then, by union bound we know that
\begin{equation}\label{eq:bound-max-normal}
\probs{\max_{i\in[m]\setminus\B}\{ \abss{ \bar{x}^i - \mu} \} \ge s} \le 2(1-\alpha)m\exp\{-n\min\{\frac{s}{2v}, \frac{s^2}{2v^2}\}\}.
\end{equation}
We proceed to analyze the trimmed mean of means estimator. To simplify notation, we define $\M = [m]\setminus\B$ as the set of all normal worker machines, $\U\subseteq[m]$ as the set of all untrimmed machines, and $\T\subseteq[m]$ as the set of all trimmed machines. The trimmed mean of means estimator simply computes
$$
\trim_\beta\{\bar{x}^i : i\in[m]\}  = \frac{1}{(1-2\beta)m}\sum_{i\in\U}\bar{x}^i.
$$
We further have
\begin{equation}\label{eq:trim-split}
\begin{aligned}
\abss{ \trim_\beta\{\bar{x}^i : i\in[m]\} - \mu  }= & \abs{ \frac{1}{(1-2\beta)m}\sum_{i\in\U}\bar{x}^i - \mu } \\
=& \frac{1}{(1-2\beta)m} \abs{ \sum_{i\in\M}(\bar{x}^i - \mu) - \sum_{i\in\M\cap\T}(\bar{x}^i - \mu) + \sum_{i\in\B\cap\U}(\bar{x}^i-\mu) } \\
\le &\frac{1}{(1-2\beta)m} \big(\abss{\sum_{i\in\M}(\bar{x}^i - \mu)}  + \abss{\sum_{i\in\M\cap\T}(\bar{x}^i - \mu)} + \abss{\sum_{i\in\B\cap\U}(\bar{x}^i-\mu) } \big)
\end{aligned}
\end{equation}
We also know that $\abss{\sum_{i\in\M\cap\T}(\bar{x}^i - \mu)} \le 2\beta m \max_{i\in\M}\{ | \bar{x}^i - \mu | \}$. In addition, since $\beta \ge \alpha$, without loss of generality, we assume that $\M\cap\T \neq \emptyset$, and then $\abss{\sum_{i\in\B\cap\U}(\bar{x}^i-\mu) } \le \alpha m \max_{i\in\M}\{ | \bar{x}^i - \mu | \}$. Then we directly obtain the desired result.

%%%%%%%%%%%%%%%%%%%%%%%%%%%%%%%%%%%%%%%%
% one-round
\section{Proof of Theorem~\ref{thm:one-round}}\label{prf:one-round}
Since the loss functions are quadratic, we denote the loss function $f(\vecw;\vecz^{i,j})$ by
$$
f(\vecw;\vecz^{i,j}) = \frac{1}{2}\vecw\tsp\matH_{i,j}\vecw + \vecp_{i,j}\tsp\vecw + c_{i,j}.
$$ 
We further define $\matH_i := \frac{1}{n}\sum_{j=1}^n\matH_{i,j}$, $\vecp_i := \frac{1}{n}\sum_{j=1}^n\vecp_{i,j}$, and $c_i := \frac{1}{n} \sum_{j=1}^n c_{i,j}$. Thus the empirical risk function on the $i$-th machine is
$$
F_i(\vecw) = \frac{1}{2}\vecw\tsp\matH_i\vecw + \vecp_i\tsp\vecw + c_i.
$$
Then, for any worker machine $i\in[m]\setminus\B$, $ \widehat{\vecw}^i = -\matH_i^{-1}\vecp_i $. In addition, the population risk minimizer is $\vecw^* = -\matH_F^{-1}\vecp_F$.
We further define $\matU_{i,j} := \matH_{i,j} -\matH_F$, $\matU_i = \matH_i - \matH_F$, $\vecv_{i,j}=\vecp_{i,j} - \vecp_F$, and $\vecv_{i}  = \vecp_i - \vecp_F$. Then
$$
\widehat{\vecw}^i  = -(\matU_i + \matH_F)^{-1}(\vecv_i + \vecp_F).
$$
Let $\vece_k$ be the $k$-th vector in the standard basis, i.e., the $k$-th column of the $d\times d$ identity matrix. We proceed to study the distribution of the $k$-th coordinate of $\widehat{\vecw}^i - \vecw^*$, $i\in[m]\setminus\B$, i.e.,
$$
\widehat{w}^i_k - w^*_k = \vece_k\tsp \matH_F^{-1}\vecp_F - \vece_k\tsp (\matU_i + \matH_F)^{-1}(\vecv_i + \vecp_F).
$$
Similar to the proof of Theorem~\ref{thm:main-gd-sc}, we need to obtain a Berry-Esseen type bound for $\widehat{w}^i_k - w^*_k$.
However, here, $\widehat{w}^i_k$ is not a sample mean of $n$ i.i.d. random variables, and thus we cannot directly apply the vanilla Berry-Esseen bound. Instead, we apply the following bound in~\cite{pinelis2016optimal} on functions of sample means.
\begin{theorem}[Theorem 2.11 in~\cite{pinelis2016optimal}, simplified]\label{thm:berry-esseen-function}
Let $\mathcal{X}$ be a Hilbert space equipped with norm $\|\cdot\|$. Let $f:\mathcal{X}\rightarrow\R$ be a function on $\mathcal{X}$. Suppose that there exists linear functions $\ell:\mathcal{X}\rightarrow\R$, $\theta>0$, $M_\theta>0$ such that
\begin{equation}\label{eq:linear-approx}
\abss{f(X) - \ell(X)} \le \frac{M_\theta}{2}\|X\|^2,~\forall~\|X\|\le \theta.
\end{equation}
Suppose that there is a probability distribution $\mathcal{D}_X$ over $\mathcal{X}$, and let $X,X_1,X_2,\ldots,X_n$ be i.i.d. random variables drawn from $\mathcal{D}_X$. Assume that $\EXPS{X}=0$, and define
$$
\widetilde{\sigma} := (\EXPS{\ell(X)^2})^{1/2},\quad \nu_p := (\EXPS{\norms{X}^p})^{1/p},p=2,3, \quad \varsigma := \frac{(\EXPS{\abss{\ell(X)}^3})^{1/3}}{\widetilde{\sigma}}.
$$
Let $\bar{X} = \frac{1}{n}\sum_{i=1}^nX_i$. Then for any $z\in\R$, we have
\begin{equation}\label{eq:berry-Esseen-general}
\abs{ \prob{\frac{f(\bar{X})}{\widetilde{\sigma}/\sqrt{n}} \le z} - \Phi(z)} \le \frac{C}{\sqrt{n}},
\end{equation}
where $C=C_0 + C_1 \varsigma^3 + (C_{20}+C_{21}\varsigma)\nu_2^2 + (C_{30}+C_{31}\varsigma)\nu_3^2+C_4$, with
\begin{equation}\label{eq:def-parameters}
\begin{aligned}
&C_0 = 0.1393,\quad C_1 = 2.3356\\
&(C_{20}, C_{21}, C_{30}, C_{31}) =  \frac{M_\theta}{2\widetilde{\sigma}} \left( 2(\frac{2}{\pi})^{1/6}, 2+\frac{2^{2/3}}{n^{1/6}},
\frac{(8/\pi)^{1/6}}{n^{1/3}},  \frac{2}{n^{1/2}}  \right) \\
&C_4 = \min\{\frac{\nu_2^2}{\theta^2n^{1/2}}, \frac{2\nu_2^3+\nu_3^3/n^{1/2}}{\theta^3n}\}.
\end{aligned}
\end{equation}
\end{theorem}
Define the function $\psi_k(\matU, \vecv):\R^{d\times d}\times\R\rightarrow\R$:
$$
\psi_k(\matU, \vecv):=\vece_k\tsp \matH_F^{-1}\vecp_F - \vece_k\tsp (\matU + \matH_F)^{-1}(\vecv + \vecp_F),
$$
and thus 
$$
\widehat{w}^i_k - w^*_k = \psi_k(\matU_i, \vecv_i) =  \psi_k(\frac{1}{n} \sum_{j=1}^n \matU_{i,j}, \frac{1}{n} \sum_{j=1}^n\vecv_{i,j}).
$$
On the product space $\R^{d\times d}\times\R$, define the element-wise inner product:
$$
\innerps{(\matU, \vecv)}{(\matX, \vecy)} = \sum_{i,j=1}^d U_{i,j}X_{i,j} + \sum_{i=1}^d v_iy_i,
$$
and thus $\R^{d\times d}\times\R$ is associated with the norm
$$
\norms{(\matU, \vecv)} = \sqrt{\fbnorms{\matU}^2 + \twonms{\vecv}^2},
$$
where $\fbnorms{\cdot}$ denotes the Frobenius norm of matrices. We then provide the following lemma on $\psi_k(\matU, \vecv)$.
\begin{lemma}\label{lem:condition-approx}
There exits a linear function $\ell_k(\matU, \vecv) = \vece_k\tsp\matH_F^{-1}\matU\matH_F^{-1}\vecp_F -\vece_k\tsp\matH_F^{-1}\vecv$ such that for any $\matU$, $\vecv$ with
$$
\fbnorms{\matU}^2 + \twonms{\vecv}^2 \le \frac{\lambda_F^2}{4},
$$
we have
$$
\abss{\psi_k(\matU, \vecv) - \ell_k(\matU, \vecv)} \le \frac{\lambda_F + 2\twonms{\vecp_F}}{\lambda_F^3} (\fbnorms{\matU}^2 + \twonms{\vecv}^2).
$$
\end{lemma}
\begin{proof}
See Appendix~\ref{prf:condition-approx}.
\end{proof}
Lemma~\ref{lem:condition-approx} tells us that the condition~\eqref{eq:linear-approx} is satisfied with $\theta = \frac{\lambda_F}{2}$ and $M_\theta = \frac{2\lambda_F + 4\twonms{\vecp_F}}{\lambda_F^3}$. For all normal worker machine $i\in[m]\setminus\B$, denote the distribution of $\matU_{i,j}$ and $\vecv_{i,j}$ by $\D_U$ and $\D_v$, respectively. Since $\widehat{w}^i_k - w^*_k =  \psi_k(\frac{1}{n} \sum_{j=1}^n \matU_{i,j}, \frac{1}{n} \sum_{j=1}^n\vecv_{i,j})$, Theorem~\ref{thm:berry-esseen-function} directly gives us the following lemma.

\begin{lemma}\label{lem:berry-esseen-inverse-mat}
Let $\matU\sim\D_U$, $\vecv\sim\D_v$, and $\ell_k(\matU, \vecv) = \vece_k\tsp\matH_F^{-1}\matU\matH_F^{-1}\vecp_F -\vece_k\tsp\matH_F^{-1}\vecv$. Define
$$
\widetilde{\sigma}_k := (\EXPS{\ell_k(\matU, \vecv)^2})^{1/2},\quad \nu_p := (\EXPS{ (\fbnorms{\matU}^2 + \twonms{\vecv}^2)^{p/2}  })^{1/p},p=2,3, \quad \varsigma_k := \frac{(\EXPS{\abss{\ell_k(\matU, \vecv)}^3})^{1/3}}{\widetilde{\sigma}_k}.
$$
Then for any $z\in\R$, $i\in[m]\setminus\B$, we have
\begin{equation}\label{eq:berry-Esseen-inverse-mat}
\abs{ \prob{\frac{ \widehat{w}^i_k - w^*_k }{\widetilde{\sigma}_k / \sqrt{n}} \le z} - \Phi(z)} \le \frac{ C_k }{\sqrt{n}},
\end{equation}
where 
$$
C_k = \widehat{C}_{0} + \widehat{C}_{1} \varsigma_k^3 + \frac{ 1 }{ \widetilde{\sigma}_k } \big[ (\widehat{C}_{20}+\widehat{C}_{21}\varsigma_k )\nu_2^2
+ (\widehat{C}_{30}+\widehat{C}_{31}\varsigma_k )\nu_3^2 \big] +\widehat{C}_{4},
$$
with
\begin{equation}\label{eq:def-parameters-inverse-mat}
\begin{aligned}
&\widehat{C}_{0} = 0.1393,\quad \widehat{C}_{1} = 2.3356\\
&(\widehat{C}_{20}, \widehat{C}_{21}, \widehat{C}_{30}, \widehat{C}_{31}) =  \frac{\lambda_F+2\twonms{\vecp_F}}{\lambda_F^3} \left( 2(\frac{2}{\pi})^{1/6}, 2+\frac{2^{2/3}}{n^{1/6}},
\frac{(8/\pi)^{1/6}}{n^{1/3}},  \frac{2}{n^{1/2}}  \right) \\
&\widehat{C}_{4} = \min\{\frac{4\nu_2^2}{\lambda_F^2n^{1/2}}, \frac{16\nu_2^3+8\nu_3^3/n^{1/2}}{\lambda_F^3n}\}.
\end{aligned}
\end{equation}
\end{lemma}

Then, we proceed to bound $\med\{\widehat{w}^i_k  : i\in[m]\} - w^*_k $, the technique is similar to what we use in the proof of Theorem~\ref{thm:uniform-bound}.
For every $z\in\R$, $k\in[d]$, define
$$
\widetilde{p}(z; k) = \frac{1}{m(1-\alpha)}\sum_{i\in[m]\setminus\B} \indi(\widehat{w}^i_k - w^*_k \le z).
$$
We have the following lemma on $\widetilde{p}(z;  k)$.
\begin{lemma}\label{lem:tildep-inverse-mat}
Suppose that for a fixed $t>0$, we have
\begin{equation}\label{eq:alpha-condition-invmat}
\alpha + \sqrt{\frac{t}{m(1-\alpha)}} + \frac{C_k}{\sqrt{n}} \le \frac{1}{2}-\epsilon,
\end{equation}
for some $\epsilon>0$. Then, with probability at least $1-4e^{-2t}$, we have
\begin{equation}\label{eq:tildep1-invmat}
\widetilde{p}\left( C_\epsilon\frac{\widetilde{\sigma}_k}{\sqrt{n}}(\alpha + \sqrt{\frac{t}{m(1-\alpha)}} + \frac{C_k}{\sqrt{n}} ) ; k \right) \ge \frac{1}{2} + \alpha,
\end{equation}
and
\begin{equation}\label{eq:tildep2-invmat}
\widetilde{p}\left( - C_\epsilon\frac{\widetilde{\sigma}_k}{\sqrt{n}}(\alpha + \sqrt{\frac{t}{m(1-\alpha)}} + \frac{C_k}{\sqrt{n}} ) ;  k \right) \le \frac{1}{2} - \alpha,
\end{equation}
where $C_\epsilon$ is defined as in~\eqref{eq:def-c-epsilon}.
\end{lemma}
\begin{proof}
The proof is essentially the same as the proof of Lemma~\ref{lem:quantile-one-d}. One can simply replace $\sigma$ in Lemma~\ref{lem:quantile-one-d} with $\widetilde{\sigma}_k$ and $0.4748\gamma(x)$ in Lemma~\ref{lem:quantile-one-d} with $C_k$. Then the same arguments still apply. Thus, we skip the details of this proof.
\end{proof}
Then, define $\widehat{p}(z; k) = \frac{1}{m}\sum_{i\in[m]} \indi(\widehat{w}^i_k - w^*_k \le z)$. Using the same arguments as in Corollary~\ref{cor:median-one-d}, we know that
$$
\widehat{p}\left( C_\epsilon\frac{\widetilde{\sigma}_k }{\sqrt{n}}(\alpha + \sqrt{\frac{t}{m(1-\alpha)}} + \frac{C_k}{\sqrt{n}} ) ; k \right) \ge \frac{1}{2},
$$
and
$$
\widehat{p}\left( - C_\epsilon\frac{\widetilde{\sigma}_k }{\sqrt{n}}(\alpha + \sqrt{\frac{t}{m(1-\alpha)}} + \frac{C_k}{\sqrt{n}} ) ; k \right) \le \frac{1}{2},
$$
which implies that $| \med\{\widehat{w}^i_k :i\in[m]\} - w^*_k | \le C_\epsilon\frac{\widetilde{\sigma}_k}{\sqrt{n}}(\alpha + \sqrt{\frac{t}{m(1-\alpha)}} + \frac{C_k}{\sqrt{n}})$. Then, let 
$$
\widetilde{\sigma} := \sqrt{\sum_{k=1}^d \widetilde{\sigma}_k^2 }= \sqrt{\EXPS{\twonms{\matH_F^{-1}(\matU\matH_F^{-1}\vecp_F-\vecv)}^2}},
$$
and $\widetilde{C} = \max_{k\in[d]}C_k$, we have with probability at least $1-4de^{-2t}$,
$$
\twonms{\med\{\widehat{\vecw}^i :i\in[m]\} - \vecw^*} \le \frac{C_\epsilon}{\sqrt{n}}\widetilde{\sigma}\big( \alpha + \sqrt{\frac{t}{m(1-\alpha)}} + \frac{\widetilde{C}}{\sqrt{n}}\big).
$$
We complete the proof by choosing $t = \frac{1}{2}\log(nmd)$.

\paragraph*{Explicit expression of $\widetilde{C}$.} To summarize, we provide an explicit expression of $\widetilde{C}$. Let $\vece_k$ be the $k$-th vector in the standard basis, i.e., the $k$-th column of the $d\times d$ identity matrix, and define $\ell_k(\matU, \vecv) :\R^{d\times d}\times \R \rightarrow \R$ as
$$
\ell_k(\matU, \vecv) = \vece_k\tsp\matH_F^{-1}\matU\matH_F^{-1}\vecp_F -\vece_k\tsp\matH_F^{-1}\vecv.
$$
Let $\matH\sim\D_H$ and $\vecp\sim\D_p$ and define
\begin{align*}
\widetilde{\sigma}_k &:= (\EXPS{\ell_k(\matH-\matH_F, \vecp - \vecp_F)^2})^{1/2},\quad \varsigma_k := \frac{(\EXPS{\abss{\ell_k(\matH-\matH_F, \vecp - \vecp_F)}^3})^{1/3}}{\widetilde{\sigma}_k}.\\
\nu_p &:= (\EXPS{ (\fbnorms{\matH-\matH_F}^2 + \twonms{\vecp - \vecp_F}^2)^{p/2}  })^{1/p},p=2,3
\end{align*}
Then, $\widetilde{C}  = \max_{k\in[d]}C_k$, with
where 
$$
C_k = \widehat{C}_{0} + \widehat{C}_{1} \varsigma_k^3 + \frac{ 1 }{ \widetilde{\sigma}_k } \big[ (\widehat{C}_{20}+\widehat{C}_{21}\varsigma_k )\nu_2^2
+ (\widehat{C}_{30}+\widehat{C}_{31}\varsigma_k )\nu_3^2 \big] +\widehat{C}_{4},
$$
with
\begin{align*}
&\widehat{C}_{0} = 0.1393,\quad \widehat{C}_{1} = 2.3356\\
&(\widehat{C}_{20}, \widehat{C}_{21}, \widehat{C}_{30}, \widehat{C}_{31}) =  \frac{\lambda_F+2\twonms{\vecp_F}}{\lambda_F^3} \left( 2(\frac{2}{\pi})^{1/6}, 2+\frac{2^{2/3}}{n^{1/6}},
\frac{(8/\pi)^{1/6}}{n^{1/3}},  \frac{2}{n^{1/2}}  \right) \\
&\widehat{C}_{4} = \min\{\frac{4\nu_2^2}{\lambda_F^2n^{1/2}}, \frac{16\nu_2^3+8\nu_3^3/n^{1/2}}{\lambda_F^3n}\}.
\end{align*}

\subsection{Proof of Lemma~\ref{lem:condition-approx}}\label{prf:condition-approx}
We use $\twonms{\cdot}$ and $\fbnorms{\cdot}$ to denote the operator norm and the Frobenius norm of matrices, respectively. We have the identity
$$
(\matI + \matA)^{-1} = \sum_{r=0}^\infty (-1)^r \matA^r, \quad \forall\twonms{\matA} <1.
$$
Then, we have for all $\matU\in\R^{d\times d}$ such that $\twonms{\matH_F^{-1}\matU} <1$,
\begin{equation}\label{eq:mat-identity}
(\matU + \matH_F)^{-1} = (\matI + \matH_F^{-1}\matU)^{-1}\matH_F^{-1} = \matH_F^{-1} - \matH_F^{-1}\matU\matH_F^{-1} + \sum_{r=2}^{\infty} (-1)^r(\matH_F^{-1}\matU)^r\matH_F^{-1}.
\end{equation}
Let us consider the set of matrices such that $\fbnorms{\matU} \le \frac{\lambda_F}{2}$. One can check that for any such matrix, we have $\twonms{\matH_F^{-1}\matU} \le \frac{1}{2}$. Let 
$$
\ell_k(\matU, \vecv) = \vece_k\tsp\matH_F^{-1}\matU\matH_F^{-1}\vecp_F -\vece_k\tsp\matH_F^{-1}\vecv.
$$
Then, we know that
\begin{equation}\label{eq:remainder-1}
\abss{\psi_k(\matU, \vecv) - \ell_k(\matU, \vecv)} = \abs{\vece_k\tsp\matH_F^{-1}\matU\matH_F^{-1}\vecv - \sum_{r=2}^\infty (-1)^r \vece_k\tsp (\matH_F^{-1}\matU)^r \matH_F^{-1}(\vecv + \vecp_F)}. 
\end{equation}
Denote the operator norm of matrices by $\twonms{\cdot}$. We further have for any $r \ge 1$,
\begin{equation}\label{eq:remainder-2}
\abss{\vece_k\tsp (\matH_F^{-1}\matU)^r \matH_F^{-1}\vecv} \le  \frac{1}{2} \twonms{\matH_F^{-1}\matU}^{r-1} (\twonms{\matH_F^{-1} \matU}^2 + \twonms{\matH_F^{-1}\vecv}^2) \le \frac{1}{2^r \lambda_F^2} (\fbnorms{\matU}^2 + \twonms{\vecv}^2),
\end{equation}
where we use the fact $\twonms{\matU} \le \fbnorms{\matU}$. In addition, for any $r \ge 2$,
\begin{equation}\label{eq:remainder-3}
\abss{\vece_k\tsp (\matH_F^{-1}\matU)^r \matH_F^{-1} \vecp_F } \le \twonms{\matH_F^{-1}\matU}^{r-2}\twonms{\matH_F^{-1}}^3\twonms{\matU}^2\twonms{\vecp_F} \le \frac{\twonms{\vecp_F}}{2^{r-2}\lambda_F^3}\fbnorms{\matU}^2.
\end{equation}
Then, we plug~\eqref{eq:remainder-2} and~\eqref{eq:remainder-3} into~\eqref{eq:remainder-1}, and obtain
$$
\abss{\psi_k(\matU, \vecv) - \ell_k(\matU, \vecv)} \le \frac{1}{\lambda_F^2} (\fbnorms{\matU}^2 + \twonms{\vecv}^2) + \frac{2\twonms{\vecp_F}}{\lambda_F^3}\fbnorms{\matU}^2,
$$
which completes the proof.

%%%%%%%%%%%%%%%%%%%%%%%%%%%%%%%%%%%
% Lower bound
\section{Proof of Observation~\ref{obs:lower-bound}}\label{prf:lower-bound}
This proof is essentially the same as the lower bound in the robust mean estimation literature~\citep{chen2015robust,lai2016agnostic}. We reproduce this result for the purpose of completeness. For a $d$ dimensional Gaussian distribution $P = \N(\vecmu, \sigma^2\mat{I})$, we denote by $P^n$ the joint distribution of $n$ i.i.d. samples of $P$. Obviously $P^n$ is equivalent to a $dn$ dimensional Gaussian distribution $\N(\vecmu^+, \sigma^2\mat{I})$, where $\vecmu^+\in\R^{dn}$ is a vector generated by repeating $\vecmu$ $n$ times, i.e., $\vecmu^+ = [\vecmu\tsp~\vecmu\tsp~\cdots~\vecmu\tsp]\tsp$.

We show that for two $d$ dimensional distributions $P_1 = \N(\vecmu_1, \sigma^2\mat{I})$ and $P_2 = \N(\vecmu_2, \sigma^2\mat{I})$, there exist two $dn$ dimensional distributions $Q_1$ and $Q_2$ such that
\begin{equation}\label{eq:no-distinguish}
(1-\alpha) P_1^n + \alpha Q_1 = (1-\alpha) P_2^n + \alpha Q_2.
\end{equation}
If this happens, then no algorithm can distinguish between $P_1$ and $P_2$. Let $\phi_1$ and $\phi_2$ be the PDF of $P_1^n$ and $P_2^n$, respectively. Let $\vecmu_1$ and $\vecmu_2$ be such that the total variation distance between $P_1^n$ and $P_2^n$ is
$$
\frac{1}{2}\int \onenms{\phi_1 - \phi_2} = \frac{\alpha}{1-\alpha}.
$$
By the results of the total variation distance between Gaussian distributions, we know that
\begin{equation}\label{eq:dist-mu1-mu2}
\twonms{\vecmu_1^+ - \vecmu_2^+} \ge  \frac{2\alpha\sigma}{1-\alpha}.
\end{equation}
Let $Q_1$ be the distribution with PDF $\frac{1-\alpha}{\alpha}(\phi_2 - \phi_1)\indi_{\phi_2 \ge \phi_1}$ and $Q_2$ be the distribution with PDF $\frac{1-\alpha}{\alpha}(\phi_1 - \phi_2)\indi_{\phi_1 \ge \phi_2}$. One can verify that~\eqref{eq:no-distinguish} is satisfied. In this case, by the lower bound in~\eqref{eq:dist-mu1-mu2}, we get
$$
\twonms{ \vecmu_1 - \vecmu_2  } \ge \frac{2\alpha\sigma}{\sqrt{n}(1-\alpha)} \ge \frac{2\alpha\sigma}{\sqrt{n}}.
$$
This implies that for two Gaussian distributions such that $\twonms{ \vecmu_1 - \vecmu_2} = \Omega(\frac{\alpha}{\sqrt{n}})$, in the worst case it can be impossible to distinguish these two distributions due to the existence of the adversary. Thus, to estimate the mean $\vecmu$ of a Gaussian distribution in the distributed setting with $\alpha$ fraction of Byzantine machines, any algorithm that computes an estimation $\widehat{\vecmu}$ of the mean has a constant probability of error $\twonms{\widehat{\vecmu} - \vecmu} = \Omega (\frac{\alpha}{\sqrt{n}}) $.

Further, according to the standard results from minimax theory~\citep{wu2017lecture}, we know that using $\bigo(nm)$ data, there is a constant probability that $\twonms{\widehat{\vecmu} - \vecmu} = \Omega(\sqrt{\frac{d}{nm}})$. Combining these two results, we know that $\twonms{\widehat{\vecmu} - \vecmu} = \Omega(\frac{\alpha}{\sqrt{n}} + \sqrt{\frac{d}{nm}})$.

\end{document}